\documentclass[11pt]{article}

\usepackage{fullpage}
\usepackage{natbib}
\usepackage{algorithm}
\usepackage[noend]{algorithmic}
\usepackage{amsmath,amsthm,amsfonts,amssymb}
\usepackage{amsmath}
\usepackage{hyperref}
\usepackage{color}
\usepackage{mathrsfs}
\usepackage{enumitem}
\usepackage{bm}
\usepackage{multirow}
\usepackage{booktabs}
\usepackage{makecell}
\usepackage{graphicx}
\usepackage{subfigure}
\usepackage{caption}
\usepackage{thmtools}
\usepackage{thm-restate}
\usepackage{hhline}
\usepackage[table]{xcolor}
\definecolor{light-gray}{gray}{0.9}

\newcommand{\defeq}{\mathrel{\mathop:}=}

\newcommand{\order}[1]{\mathcal{O}(#1)}

\newcommand{\vect}[1]{\ensuremath{\mathbf{#1}}}
\newcommand{\mat}[1]{\ensuremath{\mathbf{#1}}}
\newcommand{\dd}{\mathrm{d}}
\newcommand{\grad}{\nabla}
\newcommand{\hess}{\nabla^2}

\newcommand{\norm}[1]{\|{#1} \|}
\newcommand{\fnorm}[1]{\|{#1} \|_{\text{F}}}

\newcommand{\trans}{^{\top}}
\newcommand{\poly}{\mathrm{poly}}

\newcommand{\proj}{\mathcal{P}}

\newcommand{\tlO}{\mathcal{\tilde{O}}}

\newcommand{\tlTheta}{\tilde{\Theta}}

\newcommand{\N}{\mathbb{N}}
\newcommand{\R}{\mathbb{R}}
\renewcommand{\S}{\mathbb{S}}

\newcommand{\E}{\mathbb{E}}
\renewcommand{\Pr}{\mathbb{P}}

\newcommand{\F}{\mathcal{F}}

\newcommand{\A}{\mat{A}}
\newcommand{\B}{\mat{B}}

\newcommand{\I}{\mat{I}}

\newcommand{\X}{\mat{X}}

\newcommand{\e}{\vect{e}}
\renewcommand{\u}{\vect{u}}
\renewcommand{\v}{\vect{v}}
\newcommand{\w}{\vect{w}}
\newcommand{\x}{\vect{x}}
\newcommand{\y}{\vect{y}}

\newcommand{\g}{\vect{g}}
\newcommand{\zero}{\vect{0}}

\newcommand{\fE}{\mathfrak{E}}

\renewcommand{\H}{\mathcal{H}}
\newcommand{\cN}{\mathcal{N}}
\newcommand{\cD}{\mathcal{D}}

\newcommand{\nn}{\nonumber}

\usepackage{times}

\newtheorem{theorem}{Theorem}
\newtheorem{lemma}[theorem]{Lemma}
\newtheorem{corollary}[theorem]{Corollary}
\newtheorem{remark}[theorem]{Remark}

\newtheorem{proposition}[theorem]{Proposition}
\theoremstyle{definition}
\newtheorem{definition}[theorem]{Definition}
\newtheorem{condition}[theorem]{Condition}
\newtheorem{assumption}{Assumption}

\newcommand\blfootnote[1]{%
	\begingroup
	\renewcommand\thefootnote{}\footnote{#1}%
	\addtocounter{footnote}{-1}%
	\endgroup
}

\begin{document}

\title{\textbf{On Nonconvex Optimization for Machine Learning: Gradients, Stochasticity, and Saddle Points}}



\author{Chi Jin \\ University of California, Berkeley \\ \texttt{chijin@cs.berkeley.edu}
	\and 
	Praneeth Netrapalli \\ Microsoft Research, India \\ \texttt{praneeth@microsoft.com}
	\and
	Rong Ge \\ Duke University \\
	\texttt{rongge@cs.duke.edu}
	\and
	Sham M. Kakade \\ University of Washington, Seattle \\
	\texttt{sham@cs.washington.edu}
	\and
	Michael I. Jordan \\ University of California, Berkeley \\ \texttt{jordan@cs.berkeley.edu}}

\maketitle

\newcommand{\cnote}{\textcolor[rgb]{1,0,0}{C: }\textcolor[rgb]{1,0,1}}

\newcommand{\ball}{\mathbb{B}}
\newcommand{\Ns}{\mathfrak{N}}
\newcommand{\Ms}{\mathfrak{M}}
\newcommand{\ugrad}{\mathscr{G}}
\newcommand{\ufun}{\mathscr{F}}
\newcommand{\uspace}{\mathscr{S}}
\newcommand{\utime}{\mathscr{T}}
\newcommand{\qa}{\vect{q}_{h}}
\newcommand{\qb}{\vect{q}_{sg}}
\newcommand{\qc}{\vect{q}_{p}}
\newcommand{\p}{\vect{p}}
\newcommand{\q}{\vect{q}}
\newcommand{\coef}{\alpha}
\newcommand{\coefB}{\beta}
\newcommand{\logt}{\iota}

\newcommand{\lt}{\chi}
\newcommand{\pmat}[1]{\begin{pmatrix} #1 \end{pmatrix}}
\newcommand{\modify}[1]{#1 '}
\newcommand{\dif}[1]{\hat{#1}}
\newcommand{\la}{\langle}
\newcommand{\ra}{\rangle}
\renewcommand{\Im}{\mathrm{Im}}
\newcommand{\In}{\mathbb{I}}
\newcommand{\cXe}{\mathcal{X}_{\text{escape}}}
\newcommand{\cXs}{\mathcal{X}_{\text{stuck}}}

\newcommand{\subG}{\text{sub-Gaussian}}
\newcommand{\subE}{\text{sub-Exp}}
\newcommand{\nSG}{\text{nSG}}

\newcommand{\EFSP}{$\epsilon$-first-order stationary point}
\newcommand{\ESSP}{$\epsilon$-second-order stationary point}

\begin{abstract}

Gradient descent (GD) and stochastic gradient descent (SGD) are the workhorses of large-scale machine learning.  While classical theory focused on analyzing the performance of these methods in \emph{convex} optimization problems, the most notable successes in machine learning have involved \emph{nonconvex} optimization, and a gap has arisen between theory and practice.  Indeed, traditional analyses of GD and SGD show that both algorithms converge to stationary points efficiently.  But these analyses do not take into account the possibility of converging to saddle points.  More recent theory has shown that GD and SGD can avoid saddle points, but the dependence on dimension in these analyses is polynomial.  For modern machine learning, where the dimension can be in the millions, such dependence would be catastrophic.  We analyze perturbed versions of GD and SGD and show that they are truly efficient---their dimension dependence is only polylogarithmic.  Indeed, these algorithms converge to second-order stationary points in essentially the same time as they take to converge to classical first-order stationary points.
\blfootnote{A preliminary version of this paper, with a subset of the results that are presented here, was presented at ICML 2017 and appeared in the proceedings as~\cite{jin2017escape}.}

\end{abstract}


\section{Introduction}\label{sec:intro}

One of the principal discoveries in machine learning in recent years is an empirical one---that simple algorithms often suffice to solve difficult real-world learning problems. Machine learning algorithms generally arise via formulations as optimization problems, and, despite a massive classical toolbox of sophisticated optimization algorithms and a major modern effort to further develop that toolbox, the simplest algorithms---gradient descent, which dates to the 1840s~\citep{cauchy1847} and stochastic gradient descent, which dates to the 1950s~\citep{robbins1951stochastic}---reign supreme in machine learning.

This empirical discovery is appealing in many ways.  First, at the scale of modern machine learning applications---often involving many millions of data points and millions of parameters---complex algorithms are generally infeasible, so that the only hope is that simple algorithms might be not only feasible but successful.  Second, simple algorithms are easier to implement, debug, and maintain.  Third, as the field of machine learning transforms into a real-world engineering discipline, it will be necessary to develop solid theoretical foundations for entire systems that employ machine learning algorithms at their core, and such an effort seems
less daunting if the basic ingredients are simple.

These developments were presaged and supported by optimization researchers such as Nemirovskii, Nesterov, and Polyak who, from the 1960s until the present day, have pursued an in-depth study of first-order, gradient-based algorithms, developing novel
algorithms and accompanying theory~\citep{nemirovskii1983problem, nesterov1998introductory, polyak1963gradient}.  Their results have included lower bounds and algorithms that achieve those lower bounds.  This line of work has made clear that even simple algorithms require delicate theoretical treatment when they are studied in large-scale settings.  Thus, much of the focus has been on the setting of convex optimization where many of the complexities have been stripped away. This has allowed the development of an elegant theory, and has provided a solid jumping-off point for further analysis that has brought additional computational constraints into play---including distributed platforms, fault tolerance, communication bottlenecks, and asynchronous computation~\citep{recht2011hogwild, zhang2012communication, SmithEtAL2018}.

The most notable machine-learning success stories, however, have generally involved nonconvex optimization formulations, and a gap has arisen between theory and practice.  Attempts to fill this gap include \citet{choromanska2014loss}
in the setting of learning multi-layer neural networks,~\citet{bandeira2016low,mei2017solving}
for synchronization and MaxCut,~\citet{boumal2016non} for smooth semidefinite programs,
~\citet{bhojanapalli2016global} for matrix sensing,~\citet{ge2016matrix} for matrix
completion, and~\citet{ge2017no} for robust principal component analysis.  But there remains a need to develop general theory that relates the convergence of machine learning algorithms to geometry and dynamics.

Most of the theory in the optimization literature has focused on the relationship between the number of iterations of the algorithm and a suitable notion of accuracy. Dimension is often neglected in such analyses, in part because in the convex setting even the simplest algorithms, including gradient descent, are provably independent of dimension. In developing algorithmic theory for nonconvex optimization formulations of machine learning problems, however, it is critically important to study iteration complexity as a function of of dimension, which can be in the millions.  Moreover, we cannot resort to asymptotics---we are interested in problems at all scales.

In the current paper we have three goals.  The first is to show that significant progress has been made in recent years in the theoretical analysis of algorithms for nonconvex machine learning.  The second is to extend that line of analysis to handle both stochastic and non-stochastic algorithms in a single framework.  In both cases we upper bound the iteration complexity as a function of both accuracy and dimension. The third is to exhibit a simple proof that exposes the core of the phenomenon that determines the dimension dependence.

Nonconvex optimization problems are intractable in general.  Progress has been in machine learning by noting that in many problems the principal difficulty is not local minima, either because there are no spurious local minima (we review a list of such problems in Section~\ref{sec:prelims}) or because empirical work has shown that the local minima that are found by local gradient-based algorithms tend to be effective in terms of the ultimate goal of machine learning, which is performance on a test set.  The problem then becomes one of avoiding saddle points, which are ubiquitous in machine learning architectures.  Saddle points slow down gradient-based
algorithms and in millions of dimensions they are potentially a major bottleneck for such algorithms.  The theoretical problem becomes that of characterizing the iteration complexity of avoiding saddle points, as a function of target accuracy
and dimension.

We briefly mention some of the most relevant theoretical context for this problem here, providing a more thorough review of related work in Section~\ref{sec:related} and in the Appendix. \citet{lee2016gradient} showed that gradient descent, under random initialization or
with perturbations, asymptotically avoids saddle points with probability one. ~\cite{ge2015escaping} provided a more quantitative (nonasymptotic) characterization of gradient descent augmented with a suitable perturbation, showing that its convergence rate in the presence of saddle points is upper bounded by an expression of the form $\poly(d, \epsilon^{-1})$, where $d$ is the dimension and $\epsilon$ is the accuracy.  While these convergence results are inspiring, they are significantly worse than the convergence of gradient descent in the convex setting, or its convergence in the nonconvex setting where convergence to a saddle point is
not excluded---in both cases the rate is independent of $d$---and they do not seem to accord with the empirical success of gradient-based methods in high-dimensional problems.  Thus we ask whether these results, which are upper bounds, can be improved.

The current paper provides a positive answer to this question. We shows that suitably-perturbed versions of gradient descent and stochastic gradient descent escape saddle points in a number of iterations that is only polylogarithmic in dimension.  More technically, defining a notion of \emph{$\epsilon$-second-order stationarity} (see Section~\ref{sec:prelims}), which rules out saddle points, to be contrasted with classical \emph{$\epsilon$-first-order stationarity}, which simply means near vanishing of the gradient, and which therefore does not rule out saddle points, we show that:
\begin{itemize}
\item Perturbed gradient descent (PGD) finds $\epsilon$-second-order stationary points in $\tlO(\epsilon^{-2})$ iterations, where $\tlO(\cdot)$ hides only absolute constants and poylogarithmic factors. Compared to the $\order{\epsilon^{-2}}$ iterations required by gradient descent (GD) to find first-order stationary points \citep{nesterov1998introductory}, this involves only additional \emph{polylogarithmic} factors in $d$.
\item In the stochastic setting where stochastic gradients are Lipschitz, perturbed stochastic gradient descent (PSGD) finds $\epsilon$-second-order stationary points in $\tlO(\epsilon^{-4})$ iterations. Compared to the $\order{\epsilon^{-4}}$ iterations required by stochastic gradient descent (SGD) to find first-order stationary points \citep{ghadimi2013stochastic}, this again incurs overhead that is only \emph{polylogarithmic} in $d$.
\item When stochastic gradients are not Lipschitz, PSGD finds $\epsilon$-second-order stationary points in $\tlO(d\epsilon^{-4})$ iterations---this involves an additional \emph{linear} factor in $d$.
\end{itemize}


\renewcommand{\arraystretch}{1.5}
\begin{table*}[t]
	\centering
	\begin{tabular}
	{|>{\centering\arraybackslash}m{1in} |>{\centering\arraybackslash}m{1.9in} |>{\centering\arraybackslash}m{0.8in} | >{\centering\arraybackslash}m{1.8in}|}	
	\hline 
	\textbf{Setting} & \textbf{Algorithm} & \textbf{Iterations} & \textbf{Guarantees}\\
	\hhline{|====|}
	\multirow{2}{1in}{\centering Non-stochastic} & GD \citep{nesterov2000squared} & $\order{\epsilon^{-2}}$ & first-order stationary point\\
	\hhline{|~---|}
	&\cellcolor{light-gray}  \textbf{PGD} & \cellcolor{light-gray}  $\tlO(\epsilon^{-2})$ & \cellcolor{light-gray}  second-order stationary point \\
	\hhline{|====|}
	\multirow{3}{1in}{\centering Stochastic} & SGD \citep{ghadimi2013stochastic} & $\order{\epsilon^{-4}}$ & first-order stationary point\\
	\hhline{|~---|}
	&\cellcolor{light-gray}\textbf{PSGD} (\emph{with} Assumption \ref{assump:SG_Lip}) & \cellcolor{light-gray} $\tlO(\epsilon^{-4})$ & \cellcolor{light-gray} second-order stationary point\\
	\hhline{|~---|}
	&\cellcolor{light-gray}\textbf{PSGD} (\emph{no} Assumption \ref{assump:SG_Lip}) & \cellcolor{light-gray} $\tlO(d\epsilon^{-4})$ & \cellcolor{light-gray} second-order stationary point\\
	\hline
	\end{tabular}
	\caption{A high level summary of the results of this paper and their comparison to prior state of the art for GD and SGD algorithms. This table only highlights the dependences on $d$ and $\epsilon$. See Section~\ref{sec:intro} for a description of these results. See Section \ref{sec:related} and Appendix~\ref{app:table} for a more detailed comparison with other related works.
	}
	\label{tab:results}
\end{table*}

\subsection{Related work}
\label{sec:related}

In this section we discuss related work on convergence guarantees for finding second-order stationary points. Some key comparisons are summarized in Table~\ref{tab:results}, and an augmented table is provided in Appendix \ref{app:table}.

\paragraph{Non-stochastic settings.}
Classical approaches to finding second-order stationary points assume access to second-order information, in particular the Hessian matrix of second derivatives. Examples of such approaches include the cubic regularization method~\citep{nesterov2006cubic} and trust-region methods~\citep{curtis2014trust}, both of which require $\order{\epsilon^{-1.5}}$ queries of gradients and Hessians. This favorable convergence rate is, however, obtained at a high cost per iteration, owing to the fact that Hessian matrices scale quadratically with respect to dimension. In practice researchers have turned to first-order methods, which only utilize gradients and are therefore are substantially cheaper per iteration.

Before turning to pure first-order algorithms, we mention a line of research that is based on the assumption of a Hessian-vector product oracle~\citep{carmon2016accelerated, agarwal2017finding}.  For architectures such as deep neural networks, Hessian-vector products can be computed efficiently via automatic differentiation, and it is thus possible to obtain much of the effect of second-order methods with a complexity closer to that of first-order methods. Indeed, the algorithms have convergence rates of $\tlO(\epsilon^{-1.75})$ gradient queries~\citep{carmon2016accelerated, agarwal2017finding}. Their implementation, however, involved nested loops, and accordingly a concern with the setting of hyperparameters. Thus despite the favorable convergence rate, these algorithms have not yet found their way into practical implementations.

Given the preference among practitioners for simple, single-loop algorithms, and the striking empirical successes obtained with such algorithms, it is  important to pin down the theoretical properties of such algorithms.  In such analyses, the results for second-order and Hessian-vector algorithms serve as baselines. Of key interest is the convergence rate not merely as a function of the accuracy $\epsilon$ but also as a function of the dimension $d$.  Indeed, while second-order algorithms can use the structure of the Hessian to readily avoid unhelpful directions even in high-dimensional spaces. Without the Hessian there is a concern that algorithms may scale poorly as a function of dimension.

\citet{ge2015escaping} and \citet{levy2016power} studied simple variants of gradient descent, and found that while it has favorable scaling in terms of $\epsilon$, it requires $\poly(d)$ gradient queries to find second-order stationary points.  These results are only upper bounds, however.  The practical success of gradient descent suggests that the $\poly(d)$ scaling may be overly pessimistic.  Indeed, in an early version of the results presented here,~\citet{jin2017escape} showed that a simple perturbed version of gradient descent finds second-order stationary points in $\tlO(\epsilon^{-2})$ gradient queries, paying only a logarithmic overhead compared to the rate associated with finding first-order stationary points. This result is summarized in the first two lines of Table~\ref{tab:results}.  In followup work, \citet{jin2017accelerated} show that a perturbed version of celebrated Nesterov's accelerated gradient descent \citep{nesterov1983method} enjoys a faster convergence rate of $\tlO(\epsilon^{-1.75})$, again with logarithmic dimension dependence.

\paragraph{Stochastic setting with Lipschitz  gradient.}
 We now turn to the setting in which the learning algorithm only has access to stochastic gradients and where the stochastic is extrinsic; i.e., not under the control of the algorithm. Most existing work assumes that the stochastic gradients are Lipschitz (or equivalently that the underlying functions are gradient-Lipschitz, see Assumption \ref{assump:SG_Lip}). Under this assumption, and an additional \emph{Hessian-vector product} oracle,~\cite{allen2018natasha,zhou2018finding,tripuraneni2018stochastic} designed algorithms that have an iteration complexity of $\tlO(\epsilon^{-3.5})$.~\cite{xu2018first,allen2017neon2} obtain similar results without the requirement of a Hessian-vector product oracle. The sharpest rates in this category have been obtained by \citet{fang2018spider} and \citet{zhou2019stochastic}, who show that the iteration complexity can be further reduced to $\tlO(\epsilon^{-3})$. Again, however, this line of works consists of double-loop algorithms, and it remains unclear whether they will have an impact on practice.

 Among single-loop algorithms that are simple variants of stochastic gradient descent (SGD), \cite{ge2015escaping} showed that a particular variant has an iteration complexity for finding second-order stationary points that is upper bounded by $d^4 \poly(\epsilon^{-1})$. \cite{daneshmand2018escaping} presented an alternative variant of SGD and showed that---if the variance of the stochastic gradient along the escaping direction of saddle points is at least $\gamma$ for all saddle points---then the algorithm finds second-order stationary points in $\tlO(\gamma^{-4}\epsilon^{-5})$ iterations. In general, however, $\gamma$ scales as $1/d$, which implies a complexity of $\tlO(d^4\epsilon^{-5})$. 
 
 In the current paper, we demonstrate that a simple perturbed version of SGD achieves a convergence rate of $\tlO(\epsilon^{-4})$, which matches the speed of SGD to find a first-order stationary point up to polylogarithmic factors in dimension. Concurrent to our work, \cite{fang2019} analyzed SGD with averaging over last few iterates, and obtained a faster convergence rate of $\tlO(\epsilon^{-3.5})$.



\paragraph{General stochastic setting.}
There is significantly less work in the general setting in which the stochastic gradients are no longer guaranteed to be Lipschitz. In fact, only the results of~\citet{ge2015escaping} and~\citet{daneshmand2018escaping} apply here, and both of them require at least $\Omega(d^4)$ gradient queries to find second-order stationary points. The current paper brings this dependence down to linear dimension dependence.  See the last three lines in Table~\ref{tab:results} for a summary of the results in the stochastic case.

\paragraph{Other settings.} Finally, there are also several recent results in the setting in which objective functions can be written as a finite sum of individual functions.  We refer readers to \citet{reddi2017generic,allen2017neon2} and \citet{lei2018scsg} and the references therein for further reading.

\subsection{Organization}
In Section~\ref{sec:prelims}, we review some algorithmic and mathematical preliminaries. Section~\ref{sec:sosp} presents several examples of nonconvex problems in machine learning, demonstrating how second-order stationarity can ensure approximate global optimality. In Section~\ref{sec:result}, we present the algorithms that we analyze and present our main theoretical results for perturbed GD and SGD. In Section~\ref{sec:proof}, we present the proof for the non-stochastic case (perturbed GD), which illustrates some of our key ideas. The proof for the stochastic setting is presented in the Appendix. We conclude in Section~\ref{sec:conc}.

\section{Background}
\label{sec:prelims}

In this section, we introduce our notation and present definitions and assumptions.  We also overview existing results in nonconvex optimization, in both the deterministic and stochastic settings.

\subsection{Notation}
We use bold upper-case letters $\A, \B$ to denote matrices and bold lower-case letters $\x, \y$ to denote vectors. For vectors we use $\norm{\cdot}$ to denote the $\ell_2$-norm, and for matrices we use $\norm{\cdot}$ and $\fnorm{\cdot}$ to denote spectral (or operator) norm and Frobenius norm respectively. We use $\lambda_{\min}(\cdot)$ to denote the smallest eigenvalue of a matrix.
For a function $f: \R^{d} \rightarrow \R$, we use $\grad f$ and $\hess f$ to denote the gradient and Hessian, and $f^\star$ to denote the global minimum of function $f$. We use notation $\mathcal{O}(\cdot), \Theta(\cdot), \Omega(\cdot)$ to hide only absolute constants which do not depend on any problem parameter, and notation $\tlO(\cdot), \tilde{\Theta}(\cdot), \tilde{\Omega}(\cdot)$ to hide absolute constants and factors that are only polylogarithmically dependent on all problem parameters. 


\subsection{Nonconvex optimization and gradient descent}
In this paper, we are interested in solving general unconstrained optimization problems of the form:
\begin{align*}
	\min_{\x \in \R^d} f(\x),
\end{align*}
where $f$ is a smooth function which can be nonconvex. In particular we assume that $f$ has Lipschitz gradients and Lipschitz Hessians, which ensures that the gradient and Hessian can not change too rapidly.

\begin{definition}\label{def:smooth}
	A differentiable function $f$ is \textbf{$\ell$-gradient Lipschitz} if:
	\begin{equation*}
	\norm{\grad f(\x_1) - \grad f(\x_2)} \le \ell \norm{\x_1 - \x_2} \quad \forall \; \x_1, \x_2.
	\end{equation*}
\end{definition}

\begin{definition}\label{def:HessianLip}
	A twice-differentiable function $f$ is \textbf{$\rho$-Hessian Lipschitz} if:
	\begin{equation*}
	\norm{\hess f(\x_1) - \hess f(\x_2)} \le \rho \norm{\x_1 - \x_2} \quad \forall \; \x_1, \x_2.
	\end{equation*}
\end{definition}
\begin{assumption}\label{assump:GD}
The function $f$ is $\ell$-gradient Lipschitz and $\rho$-Hessian Lipschitz.
\end{assumption}

Our point of departure is the classical Gradient Descent (GD) algorithm, whose update takes following form:
\begin{equation}
\x_{t+1} = \x_{t} - \eta \grad f(\x_t), \label{eq:GD}
\end{equation}
where $\eta>0$ is a step size or learning rate.
Since the problem of finding a global optimum for general nonconvex functions is NP-hard, the classical literature in optimization has resorted to a local surrogate---first-order stationarity.



\begin{definition}\label{def:0FOSP}
For a differentiable function $f$, $\x$ is a \textbf{first-order stationary point} if $\grad f(\x) = \zero$.
\end{definition}

\begin{definition}\label{def:FOSP}
For a differentiable function $f$, $\x$ is an \textbf{$\epsilon$-first-order stationary point} if $\norm{\grad f(\x)}\le \epsilon$.
\end{definition}

It is of major importance that gradient descent converges to a first-order stationary point in a number of iterations that is independent of dimension.
This fact, referred to as ``dimension-free convergence'' in the optimization literature, is captured in the following classical theorem.

\begin{theorem}[\citep{nesterov1998introductory}]  \label{thm:classic_GD}
For any $\epsilon>0$, assume the function $f(\cdot)$ is $\ell$-gradient Lipschitz, and set the step size as $\eta = 1/\ell$.
Then, the gradient descent algorithm in Eq.~\eqref{eq:GD} will visit an $\epsilon$-stationary point at least once in the following number of iterations:
\begin{equation*}
\frac{\ell (f(\x_0) - f^\star) }{\epsilon^2}.
\end{equation*}
\end{theorem}

Note that in this formulation, the last iterate is not guaranteed to be a stationary point. However, it is not hard to figure out which iterate is the stationary point by calculating the norm of the gradient at every iteration.

A first-order stationary point can be a local minimum, a local maximum or even a saddle point:
\begin{definition}
For a differentiable function $f$, a stationary point $\x$ is a
\begin{itemize}
\item \textbf{local minimum}, if there exists $\delta >0 $ such that $f(\x) \le f(\y) $ for any $\y$ with $\norm{\y - \x} \le \delta$.
\item \textbf{local maximum}, if there exists $\delta >0 $ such that $f(\x) \ge f(\y) $ for any $\y$ with $\norm{\y - \x} \le \delta$.
\item \textbf{saddle point}, otherwise.
\end{itemize}
\end{definition}

For minimization problems, both saddle points and local maxima are clearly undesirable. Our focus will be ``saddle points,'' although our results also apply directly to local maxima as well. Unfortunately, distinguishing saddle points from local minima for smooth functions is still NP-hard in general \citep{nesterov2000squared}. To avoid these hardness results, we focus on a subclass of saddle points.


\begin{definition}
For a twice-differentiable function $f$, $\x$ is a \textbf{strict saddle point} if $\x$ is a stationary point and
$\lambda_{\min} (\hess f(\x)) < 0$.
\end{definition}

A generic saddle point must satisfy that $\lambda_{\min} (\hess f(\x)) \le 0$. Being ``strict'' simply rules out the case where $\lambda_{\min} (\hess f(\x)) = 0$. We reformulate our goal as that of finding stationary points that are not strict saddle points.
\begin{definition}\label{def:0SOSP}
For twice-differentiable function $f(\cdot)$, $\x$ is a \textbf{second-order stationary point} if
\begin{equation*}
\grad f(\x) = \zero, \quad\text{and}\quad \hess f(\x) \succeq \zero.
\end{equation*}
\end{definition}
\begin{definition}\label{def:SOSP}
	For a $\rho$-Hessian Lipschitz function $f(\cdot)$, $\x$ is an \textbf{\ESSP} if:
	\begin{equation*}
	\norm{\grad f(\x)} \le \epsilon \quad\text{and}\quad \hess f(\x) \succeq - \sqrt{\rho \epsilon} \cdot \I.
	\end{equation*}
\end{definition}
Our definition again makes use of an $\epsilon$-ball around the stationary point so that we can discuss rates, and the condition on the Hessian in Definition \ref{def:FOSP} uses the Hessian Lipschitz parameter $\rho$ to retain a single accuracy parameter and to match the units of the gradient and Hessian, following the convention of \citet{nesterov2006cubic}.

Although second-order stationarity is only a necessary condition for being a local minimum, a line of recent work in the machine learning literature shows that for many popular models in machine learning, all $\epsilon$-second-order stationary points are approximate global minima.  Thus for these models finding second-order stationary points is sufficient for solving those problems. See Section \ref{sec:sosp} for references and discussion of these results.


\subsection{Stochastic approximation}
We turn to the stochastic approximation setting, where we cannot access the exact gradient $\grad f(\cdot)$ directly. Instead for any point $\x$, a gradient query will return a stochastic gradient $\g(\x; \theta)$, where $\theta$ is a random variable drawn from a distribution $\mathcal{D}$. The key property that we assume for stochastic gradients is that they are unbiased: $\nabla f(\x) = \E_{\theta \sim \mathcal{D}}\left[\g(\x;\theta)\right]$.  That is, the average value of the stochastic gradient equals the true gradient.
In short, the update of Stochastic Gradient Descent (SGD) is:
\begin{equation} \label{eq:sgd}
\text{Sample}~~\theta_t \sim \mathcal{D}, \qquad \x_{t+1} = \x_t - \eta \grad \g(\x_t; \theta_t)
\end{equation}

Other than being an unbiased estimator of true gradient, another standard assumption on the stochastic gradients is that their variance is bounded by some number $\sigma^2$:
\begin{align*}
	\E_{\theta \sim \mathcal{D}} \left[\norm{\g(\x,\theta) - \nabla f(\x)}^2\right] \leq \sigma^2
\end{align*}
When we are interested in high-probability bounds, we make the following stronger assumption on the tail of the distribution.
\begin{assumption}\label{assump:SG} 
	For any $\x \in \R^d$, the stochastic gradient $\g(\x; \theta)$ with $\theta \sim \cD$ satisfies:
	\begin{equation*}
	\E \g(\x;\theta) = \grad f(\x), \qquad \Pr\left(\norm{\g(\x;\theta) - \grad f(\x)} \ge t \right) \le 2 \exp(-t^2/(2\sigma^2)), \qquad \forall t \in \R.
	\end{equation*}
\end{assumption}
We note this assumption is more general than the standard notion of a sub-Gaussian random vector, which assumes $\E \exp(\la \v, \X - \E\X \ra) \le \exp(\sigma^2 \norm{\v}^2 /d)$ for any $\v \in \R^d$. The latter assumption requires the distribution to be ``isotropic'' while our assumption does not. By Lemma~\ref{lem:examplesGnorm} we know that both bounded random vectors and standard sub-Gaussian random vector are special cases of our more general setting.

Prior work shows that stochastic gradient descent converges to first-order stationary points in a number of iterations that is independent of dimension.

\begin{theorem}[\citep{ghadimi2013stochastic}] \label{thm:classic_sgd}
For any $\epsilon, \delta > 0$, assume that the function $f$ is $\ell$-gradient Lipschitz, that the stochastic gradient $\g$ satisfies Assumption~\ref{assump:SG}, 
and let the step size scale as $\eta = \tilde{\Theta}(\ell^{-1} (1 + \sigma^2/\epsilon^2)^{-1})$. Then, with probability at least $1-\delta$, 
stochastic gradient descent will visit an $\epsilon$-stationary point at least once in the following number of iterations:
\begin{equation*}
\tlO\left(\frac{\ell (f(\x_0) - f^\star) }{\epsilon^2} \left(1 + \frac{\sigma^2}{\epsilon^2}\right)\right).
\end{equation*}
\end{theorem}




\section{On the Sufficiency of Second-Order Stationarity}
\label{sec:sosp}

In this section we show that for a wide class of nonconvex problems in machine learning and signal processing, all second-order stationary points are global minima.  Thus, for this class of problems, finding second-order stationary points efficiently is equivalent to solving the problem.  We focus on the underlying global geometry that these problems have in common.

Problems for which second-order stationary points are global minima include tensor decomposition \citep{ge2015escaping}, dictionary learning \citep{sun2016complete}, phase retrieval \citep{sun2016geometric}, synchronization and MaxCut \citep{bandeira2016low,mei2017solving}, smooth semidefinite programs \citep{boumal2016non},
and many problems related to low-rank matrix factorization, such as matrix sensing \citep{bhojanapalli2016global},
matrix completion \citep{ge2016matrix} and robust principale component analysis \citep{ge2017no}.  In particular, 
these papers show that by adding appropriate regularization terms, and under mild conditions, there are two key geometric properties satisfied by the corresponding objective functions: (a) All local minima are global minima. There might be multiple local minima due to permutation, but they are all equally good; (b) All saddle points have at least one direction with strictly negative curvature, thus are strict saddle points. we summarize the consequences of these properties in the following proposition, for which we omit the proof as it follows essentially by definition.
\begin{proposition} \label{fact:landscape}
If a function $f$ satisfies (a) all local minima are global minima; (b) all saddle points (including local maxima) are strict saddle points, 
then all second-order stationary points are global minima.
\end{proposition}

This implies that the core problem for these nonconvex machine-learning applications is to find second-order stationary points efficiently. If we can prove that some simple variants of GD and SGD converges to second-order stationary points efficiently, then we immediately establish global convergence results for these algorithms for all of the above applications (i.e. convergence from arbitrary initialization).

Before we turn to such convergence results, consider, as an illustrative example of this class of nonconvex problems, the problem of finding the leading eigenvector of a positive semidefinite matrix $\mat{M} \in \R^{d\times d}$.  We define the following objective:
\begin{equation}\label{eq:obj}
\min_{\x \in \R^{d}} f(\x) = \frac{1}{2}\fnorm{\x\x\trans - \mat{M}}^2,
\end{equation} 
Denote the eigenvalues and eigenvectors of $\mat{M}$ as $(\lambda_i, \v_i)$ for $i = 1, \ldots, d$, and assume there is a gap between the first and second eigenvalues: $\lambda_1 > \lambda_2 \ge \lambda_3 \ge \ldots \ge \lambda_d \ge 0$. In this case, the global optimal solutions are $\x = \pm \sqrt{\lambda_1} \v_1$ giving the top eigenvector direction.

The objective function \eqref{eq:obj} is nonconvex as a function of $\x$. In order to optimize this objective via gradient-descent methods, we need to analyze the global landscape of the objective function. Its gradient and Hessian are of the form:
\begin{align*}
\grad f(\x) =& (\x\x\trans -\mat{M}) \x \\
\hess f(\x) =& \norm{\x}^2 \I  + 2\x\x\trans - \mat{M}.
\end{align*}
Therefore, all stationary points satisfy the equation $\mat{M} \x = \norm{\x}^2\x$. Thus they are $\zero$ and $\pm \sqrt{\lambda_i}\v_i$ for $i = 1, \ldots, d$. We already know that $\pm \sqrt{\lambda_1} \v_1$ are global minima, thus they are also local minima and are equivalent up to a sign difference. For the remaining stationary points $\x^\dagger$, we note that their Hessian always has strict negative curvature along the $\v_1$ direction: $\v_1\trans\hess f(\x^\dagger) \v_1 \le \lambda_2 - \lambda_1 <0$. Thus these points are strict saddle points. Having established the preconditions for Proposition \ref{fact:landscape}, we are able to conclude:
\begin{corollary}
\label{prop:geometry}
Assume that $\mat{M}$ is a positive semidefinite matrix whose top two eigenvalues are $\lambda_1 > \lambda_2 \ge 0$.  For the problem of minimizing the objective Eq.~\eqref{eq:obj}, all second-order stationary points are global optima.
\end{corollary}
Further analysis can be carried out to establish the $\epsilon$-robust version of Corollary \ref{prop:geometry}. Informally, it can be shown that under technical conditions, for polynomially small $\epsilon$, all $\epsilon$-second-order stationary point are close to global optima. We refer the readers to \cite{ge2017no} for the formal statement.




\section{Main Results}
\label{sec:result}

In this section, we present our main results on the efficiency of GD and SGD in the nonconvex setting. We first study the case where the exact gradients are accessible, where we focus on an algorithm that we refer to as \emph{Perturbed Gradient Descent} (PGD). We then turn to the stochastic setting, and present the results for Perturbed SGD and its mini-batch version.

\subsection{Non-stochastic setting}
\label{sec:non-stochastic}

We begin by considering the case in which exact gradients are available, such that GD can be implemented. For convex problems, GD is efficient, but, as can be seen in Eq.\eqref{eq:GD}, GD makes a non-zero step only when the gradient is non-zero, and thus in the nonconvex setting it will be stuck at saddle points if initialized there. We thus consider a simple variant of GD which adds randomness to the iterates at each step (Algorithm \ref{algo:PGD}).  The question is whether such a simple procedure can be efficient, particularly in terms of its dimension dependence.


\begin{algorithm}[t]
\caption{Perturbed Gradient Descent (PGD)}\label{algo:PGD}
\begin{algorithmic}
\renewcommand{\algorithmicrequire}{\textbf{Input: }}
\renewcommand{\algorithmicensure}{\textbf{Output: }}
\REQUIRE $\x_0$, step size $\eta$, perturbation radius $r$.
\FOR{$t = 0, 1, \ldots, $}
\STATE $\x_{t+1} \leftarrow \x_t - \eta (\grad f (\x_t) + \xi_t), \qquad \xi_t \sim \cN(\zero, (r^2/d) \I)$
\ENDFOR
\end{algorithmic}
\end{algorithm}

At each iteration, Algorithm \ref{algo:PGD} is almost the same as gradient descent, except it adds a small isotropic random Gaussian perturbation to the gradient. The perturbation $\xi_t$ is sampled from a zero-mean Gaussian with covariance $(r^2/d) \I$ so that $\E \norm{\xi_t}^2 = r^2$. We note that Algorithm~\ref{algo:PGD} is different from that studied in \cite{jin2017escape}, where perturbation was added only when certain conditions hold.


We now show that if we pick $r = \tilde{\Theta}(\epsilon)$ in Algorithm~\ref{algo:PGD}, PGD will find an $\epsilon$-second-order stationary point in a number of iterations that has only a polylogarithmic dependence on dimension.

\begin{theorem}\label{thm:main_GD}
 Let the function $f(\cdot)$ satisfy Assumption~\ref{assump:GD}.  Then, for any $\epsilon, \delta>0$, the PGD algorithm (Algorithm \ref{algo:PGD}), with parameters 
$\eta = \tilde{\Theta}(1/\ell)$ and $r = \tilde{\Theta}(\epsilon)$, 
 will visit an $\epsilon-$second-order stationary point at least once in the following number of iterations, with probability at least $1-\delta$:
\begin{equation*}
\tlO\left(\frac{\ell (f(\x_0) - f^\star) }{\epsilon^2} \right),
\end{equation*}
where $\tlO$ and $\tilde{\Theta}$ hide polylogarithmic factors in $d, \ell, \rho, 1/\epsilon, 1/\delta$ and $\Delta_f\defeq f(\x_0) - f^\star$.
\end{theorem}

\begin{remark}
\label{rm:output}
If we wish to output an $\epsilon$-second-order stationary point, it suffices to run PGD for double the number of iterations in Theorem \ref{thm:main_GD}.  A simple change to the proof shows that half of the iterates will be $\epsilon$-second-order stationary points in this case, so that if we output an iterate uniformly at random, with at least a constant probability it will be an $\epsilon$-second-order stationary point. 
\end{remark}

\begin{remark}
\label{rm:dist}
We have chosen the distribution of the perturbations to be Gaussian in Algorithm \ref{algo:PGD} for simplicity.  This choice is not necessary. The key properties needed for the perturbation distributions are (a) that the tail of the distribution is sufficiently light such that an appropriate concentration inequality holds, and (b) the variance in every direction is bounded below.
\end{remark}

Comparing Theorem \ref{thm:main_GD} to the classical result in Theorem \ref{thm:classic_GD}, our result shows that PGD finds second-order stationary points in almost the same time as GD finds first-order stationary points, up to only logarithmic factors. Therefore, strict saddle points are computationally benign for first-order gradient methods. 

Comparing to Theorem \ref{thm:classic_GD}, we see that Theorem \ref{thm:main_GD} makes an additional smoothness assumption.  This assumption is essential in separating strict saddle points from second-order stationary points.


\subsection{Stochastic setting}

\label{sec:stochastic}

In the stochastic approximation setting, exact gradients $\grad f(\cdot)$ are no longer available, and the algorithms only have access to unbiased stochastic gradients: $\g(\cdot;\theta)$ such that $\nabla f(\x) = \E_{\theta \sim \mathcal{D}}\left[\g(\x;\theta)\right]$.


In machine learning, the stochastic gradient $\g$ is often obtained as an exact gradient of a smooth function: $\g(\cdot; \theta) = \grad f(\cdot; \theta)$. We formalize this assumption.

\begin{assumption} \label{assump:SG_Lip}
	For any $\theta \in \text{supp}(\cD)$, $\g(\cdot; \theta)$ is $\tilde{\ell}$-Lipschitz:
    \begin{equation*}
    \norm{\g(\x_1; \theta) - \g(\x_2; \theta)} \le \tilde{\ell} \norm{\x_1 - \x_2} \quad \forall \; \x_1, \x_2.
    \end{equation*}
\end{assumption}
 In the special case where $\g(\cdot; \theta) = \grad f(\cdot; \theta)$ for some twice-differentiable  function $f(\cdot; \theta)$, Assumption \ref{assump:SG_Lip} ensures that the spectral norm of Hessian of function $f(\cdot; \theta)$ is bounded by $\tilde{\ell}$ for all $\theta$. Therefore, the stochastic Hessian also enjoys good concentration properties, which allows algorithms to find points with a second-order characterization. In contrast, when Assumption \ref{assump:SG_Lip} no longer holds, the problem of finding second-order stationary points becomes due to the lack of concentration of the stochastic Hessian. In the current paper, we treat both cases, by allowing $\tilde{\ell} = +\infty$ to encode the assertion that Assumption~\ref{assump:SG_Lip} does not hold.




\begin{algorithm}[t]
\caption{Perturbed Stochastic Gradient Descent (PSGD)}\label{algo:PSGD}
\begin{algorithmic}
\renewcommand{\algorithmicrequire}{\textbf{Input: }}
\renewcommand{\algorithmicensure}{\textbf{Output: }}
\REQUIRE $\x_0$, step size $\eta$, perturbation radius $r$.
\FOR{$t = 0, 1, \ldots, $}
\STATE sample $\theta_t \sim \cD$
\STATE $\x_{t+1} \leftarrow \x_t - \eta (\g (\x_t;\theta_t) + \xi_t), \qquad \xi_t  \sim \cN(\zero, (r^2/d) \I)$
\ENDFOR
\end{algorithmic}
\end{algorithm}

\begin{algorithm}[t]
\caption{Mini-batch Perturbed Stochastic Gradient Descent (Mini-batch PSGD)}\label{algo:Mini_PSGD}
\begin{algorithmic}
\renewcommand{\algorithmicrequire}{\textbf{Input: }}
\renewcommand{\algorithmicensure}{\textbf{Output: }}
\REQUIRE $\x_0$, step size $\eta$, perturbation radius $r$.
\FOR{$t = 0, 1, \ldots, $}
\STATE sample $\{\theta^{(1)}_t, \cdots \theta^{(m)}_t\} \sim \mathcal{D}$
\STATE $\g_t(\x_t) \leftarrow \sum_{i=1}^m \g (\x_t;\theta^{(i)}_t) /m $
\STATE $\x_{t+1} \leftarrow \x_t - \eta (\g_t(\x_t)+ \xi_t), \qquad \xi_t \sim \cN(\zero, (r^2/d) \I)$
\ENDFOR
\end{algorithmic}
\end{algorithm}


We are now ready to present a guarantee on the efficiency of PSGD (Algorithm~\ref{algo:PSGD}) for finding a second-order stationary point.
We make the following choices of parameters for Algorithm \ref{algo:PSGD}:
\begin{equation}\label{eq:SGD_para}
\eta = \tlTheta(\frac{1}{\ell \cdot \Ns}), \quad
r = \tlTheta(\epsilon \sqrt{\Ns}), \quad \text{where} \quad \Ns = 1 + \min\left\{\frac{\sigma^2}{\epsilon^2} + \frac{\tilde{\ell}^2}{\ell\sqrt{\rho\epsilon}}, ~\frac{\sigma^2d}{\epsilon^2}\right\}
\end{equation}

\begin{theorem}\label{thm:main_SGD} 
Let the function $f$ satisfy Assumption~\ref{assump:GD} and assume that the stochastic gradient $\g$ satisfies Assumption~\ref{assump:SG} (or Assumption~\ref{assump:SG_Lip} optionally).  For any $\epsilon, \delta > 0$, the PSGD algorithm (Algorithm~\ref{algo:PSGD}), with parameter $(\eta, r)$ chosen as in Eq.~\eqref{eq:SGD_para},
will visit an $\epsilon-$second-order stationary point at least once in the following number of iterations, with probability at least $1-\delta$:
\begin{equation*}
\tilde{O}\left(\frac{\ell(f(\x_0) - f^\star)}{\epsilon^2} \cdot \Ns \right).
\end{equation*}
\end{theorem}
\noindent
We note that Remark \ref{rm:output} and Remark \ref{rm:dist} apply directly to Theorem \ref{thm:main_SGD}.

Theorem \ref{thm:main_SGD} provides a result for both the scenario in which Assumption \ref{assump:SG_Lip} holds and when it does not. In the former case, taking $\epsilon$ such that $\sigma^2/\epsilon^2 \ge \tilde{\ell}^2/ (\ell\sqrt{\rho\epsilon})$, we have that $\Ns \approx 1 + \sigma^2/\epsilon^2$. Our results then show that perturbed SGD finds a second-order stationary points in $\tlO(\epsilon^{-4})$ iterations, which matches Theorem \ref{thm:classic_sgd} up to logarithmic factors. 

In the general case where Assumption \ref{assump:SG_Lip} does not hold ($\tilde{\ell} = \infty$), we have $\Ns = 1+ \sigma^2 d/\epsilon^2$, and Theorem~\ref{thm:main_SGD} guarantees that PSGD finds an $\epsilon$-second-order stationary point in $\tlO(d\epsilon^{-4})$ iterations. Comparing to Theorem \ref{thm:classic_sgd}, we see that the algorithm pays an additional cost that is linear in dimension $d$.


Finally, Theorem~\ref{thm:main_SGD} can be easily extended to the mini-batch setting, with parameters chosen as:
\begin{equation}\label{eq:SGD_minibatch_para}
\eta = \tlTheta(\frac{1}{\ell \cdot \Ms}), \quad
r = \tlTheta(\epsilon \sqrt{\Ms}), \quad \text{where} \quad \Ms = 1 + \frac{1}{m}\min\left\{\frac{\sigma^2}{\epsilon^2} + \frac{\tilde{\ell}^2}{\ell\sqrt{\rho\epsilon}}, ~\frac{\sigma^2d}{\epsilon^2}\right\}.
\end{equation}



\begin{theorem}[Mini-batch version]\label{thm:main_SGD_minibatch}
Let the function $f$  satisfy Assumption~\ref{assump:GD} and assume that the stochastic gradient $\g$ satisfies Assumption \ref{assump:SG} (or \ref{assump:SG_Lip} optionally). Then, for any $\epsilon, \delta, m > 0$, the mini-batch PSGD algorithm (Algorithm \ref{algo:Mini_PSGD}), with parameters $(\eta, r)$ chosen as in Eq.~\eqref{eq:SGD_minibatch_para}, will visit an $\epsilon-$second-order stationary point at least once in the following number of iterations, with probability at least $1-\delta$:
\begin{equation*}
\tilde{O}\left(\frac{\ell(f(\x_0) - f^\star)}{\epsilon^2} \cdot \Ms \right).
\end{equation*}
\end{theorem}

Theorem~\ref{thm:main_SGD_minibatch} says that if the mini-batch size $m$ is not too large---$m \le \Ns$, where $\Ns$ is defined in Eq.~\eqref{eq:SGD_para}---then mini-batch PSGD will reduce the number of iterations linearly, while not increasing the total number of stochastic gradient queries.




\section{Proofs}\label{sec:proof}

We provide full proofs of our main results, Theorem~\ref{thm:main_SGD} and Theorem~\ref{thm:main_SGD_minibatch}, in Appendix~\ref{app:proof}. (Theorem~\ref{thm:main_GD} follows directly from the proof of Theorem~\ref{thm:main_SGD} by setting $\sigma=0$).  These proofs require novel concentration inequalities and other tools from stochastic analysis to handle the relatively complex way in which stochasticity interacts with geometry in the neighborhood of saddle points.  In the current section we circumvent some of these complexities by presenting a conceptually straightforward proof for an algorithm that is a variant of PGD.  This algorithm, summarized in Algorithm~\ref{algo:PGDV}, removes some of the stochasticity of PGD by restricting the way in which perturbation noise is added.  The algorithm is more complex than PGD, but the proof is streamlined, allowing the core concepts underlying the full proof to be conveyed more simply.

The following theorem is the specialization of Theorem \ref{thm:main_GD} to the setting of Algorithm~\ref{algo:PGDV}.

\begin{algorithm}[t]
\caption{Perturbed Gradient Descent (Variant)}\label{algo:PGDV}
\begin{algorithmic}
\renewcommand{\algorithmicrequire}{\textbf{Input: }}
\renewcommand{\algorithmicensure}{\textbf{Output: }}
\REQUIRE $\x_0$, step size $\eta$, perturbation radius $r$, time interval $\utime$, tolerance $\epsilon$.
\STATE $t_{\text{perturb}} = 0$
\FOR{$t = 0, 1, \ldots, T $}
\IF{$\norm{\grad f(\x_t)} \le \epsilon$ and $t - t_{\text{perturb}} > \utime$}
\STATE $\x_t \leftarrow \x_t - \eta\xi_t, ~(\xi_t \sim \text{Uniform}(B_0(r))); \quad t_{\text{perturb}} \leftarrow t$
\ENDIF
\STATE $\x_{t+1} \leftarrow \x_t - \eta \grad f (\x_t)$
\ENDFOR
\end{algorithmic}
\end{algorithm}

\begin{theorem}\label{thm:main_gdez}
Let $f$ satisfy Assumption~\ref{assump:GD} and define $\Delta_f \defeq f(\x_0) - f^\star$.  For any $\epsilon, \delta>0$, the PGD (Variant) algorithm (Algorithm~\ref{algo:PGDV}), with parameters $\eta, r, \utime$ chosen as in Eq.~\eqref{eq:para_gd} and with $\logt = c \cdot \log( d \ell \Delta_f/(\rho\epsilon\delta))$, has the property that at least one half of its iterations of will be $\epsilon-$second order stationary points, after the following number of iterations, and with probability at least $1-\delta$:
\begin{equation*}
\tilde{O}\left(\frac{\ell \Delta_f }{\epsilon^2} \right).
\end{equation*}
Here $c$ is an absolute constant.
\end{theorem}



We proceed to the proof of the theorem.  We first specify the choice of hyperparameters $\eta$, $r$, and $\utime$, and two quantities $\ufun$ and $\uspace$ which are frequently used:

\begin{equation}\label{eq:para_gd}
\eta = \frac{1}{\ell}, \quad r = \frac{\epsilon}{400 \iota^3}, \quad \utime = \frac{\ell}{ \sqrt{\rho\epsilon}} \cdot \logt, 
\quad \ufun = \frac{1}{50 \logt^3} \sqrt{\frac{\epsilon^3}{\rho}}, \quad 
\uspace = \frac{1}{4\logt} \sqrt{\frac{\epsilon}{\rho}}.
\end{equation}

Our high-level proof strategy is a proof by contradiction: when the current iterate is not an $\epsilon$-second order stationary point, it must either have a large gradient or have a strictly negative Hessian, and we prove that in either case, PGD must yield a significant decrease in function value in a controlled number of iterations. Also, since the function value can not decrease more than $f(\x_0) - f^\star$, we know that the total number of iterates that are not $\epsilon$-second order stationary points can not be very large.

First, we bound the rate of decrease when the gradient is large.

\begin{lemma}[Descent Lemma]\label{lem:descent_gd}
If $f(\cdot)$ satisfies Assumption \ref{assump:GD} and $\eta \le 1/\ell$, then the gradient descent sequence $\{\x_t\}$ satisfies:
\begin{equation*}
f(\x_{t+1}) - f(\x_{t}) \le -\eta\norm{\grad f(\x_t)}^2/2.
\end{equation*}
\end{lemma}
\begin{proof}
Due to the $\ell$-gradient Lipschitz assumption, we have:
\begin{align*}
f(\x_{t+1}) & \le f(\x_t) + \la\grad f(\x_t), \x_{t+1} - \x_t \ra + \frac{\ell}{2} \norm{\x_{t+1} - \x_t}^2 \\
 &= f(\x_t)  - \eta\norm{\grad f(\x_t)}^2 + \frac{\eta^2\ell}{2} \norm{\grad f(\x_t)}^2
 \le f(\x_t) - \frac{\eta}{2}\norm{\grad f(\x_t)}^2
\end{align*}
\end{proof}

We now show that if the starting point has a strictly negative eigenvalue of the Hessian, then adding a perturbation and following by gradient descent will yield a significant decrease in function value in $\utime$ iterations.
\begin{lemma}[Escaping Saddle Points]\label{lem:escapesaddle_gd}
Assume that $f(\cdot)$ satisfies Assumption \ref{assump:GD}, $\tilde{\x}$ satisfies $\norm{\grad f(\tilde{\x})} \le \epsilon$, and $\lambda_{\min}(\hess f(\tilde{\x})) \le -\sqrt{\rho\epsilon}$. Let $\x_0 = \tilde{\x} + \eta\xi ~(\xi \sim \text{Uniform}(B_0(r)))$.  Run gradient descent starting from $\x_0$.  This yields: 
\begin{equation*}
\Pr(f(\x_\utime) - f(\tilde{\x}) \le -\ufun/2) \ge 1-\frac{\ell  \sqrt{d}}{\sqrt{\rho\epsilon}}\cdot \logt^2 2^{8-\logt} ,
\end{equation*}
where $\x_\utime$ is the $\utime^{\textrm{th}}$ gradient descent iterate starting from $\x_0$.
\end{lemma}

In order to prove this, we need to prove two lemmas, and the major simplification over~\cite{jin2017escape} comes from the following lemma which says that if function value does not decrease too much over $t$ iterations, then all iterates $\{\x_{\tau}\}_{\tau = 0}^t$ will remain in a small neighborhood of $\x_0$.

\begin{lemma}[Improve or Localize]\label{lem:improveorlocalize_gd}
Under the setting of Lemma \ref{lem:descent_gd}, for any $t \ge \tau >0$:
\begin{equation*}
\norm{\x_{\tau} - \x_0} \le \sqrt{2\eta t (f(\x_0) - f(\x_t) )}.
\end{equation*}
\end{lemma}
\begin{proof}
Given the gradient update, $\x_{t+1} = \x_t - \eta \grad f(\x_t)$, we have that for any $\tau \le t$:
\begin{align*}
\norm{\x_\tau - \x_0} & \le \sum_{\tau = 1}^t \norm{\x_\tau - \x_{\tau - 1}}
\overset{(1)}{\le} [t \sum_{\tau = 1}^t \norm{\x_\tau - \x_{\tau - 1}}^2]^{\frac{1}{2}} \\
&= [\eta^2 t\sum_{\tau = 1}^t \norm{\grad f(\x_{\tau - 1})}^2]^{\frac{1}{2}} 
\overset{(2)}\le \sqrt{2\eta t (f(\x_0) - f(\x_t))},
\end{align*}
where step (1) uses Cauchy-Schwarz inequality, and step (2) is due to Lemma \ref{lem:descent_gd}.
\end{proof}


Second, we show that the region in which GD will get remain in a small local neighborhood for at least $\utime$ iterations if initialized there (which we refer to as the ``stuck region'') is thin. We show this by tracking any pair of points that differ only in an escaping direction and are at least $\omega$ far apart. We show that at least one sequence of the two GD sequences initialized at these points is guaranteed to escape the saddle point with high probability, so that the width of the stuck region along an escaping direction is at most $\omega$.

\begin{lemma}[Coupling Sequence]\label{lem:coupleseq_gd}
Suppose $f(\cdot)$ satisfies Assumption \ref{assump:GD} and $\tilde{\x}$ satisfies $\lambda_{\min}(\hess f(\tilde{\x})) \le -\sqrt{\rho\epsilon}$.
Let $\{\x_t\}$, $\{\modify{\x}_t\}$ be two gradient descent sequences which satisfy: (1) $\max\{\norm{\x_0 - \tilde{\x}}, \norm{\modify{\x}_0 - \tilde{\x}}\} \le \eta r$; and (2) $\x_0 - \modify{\x}_0 = \eta r_0 \e_1$, where $\e_1$ is the minimum eigenvector of $\hess f(\tilde{\x})$ and $r_0 > \omega \defeq 2^{2-\logt} \ell \uspace $. Then:
\begin{equation*}
\min\{f(\x_\utime) - f(\x_0), f(\modify{\x}_\utime) - f(\modify{\x}_0)\} \le - \ufun.
\end{equation*}
\end{lemma}
\begin{proof}
Assume the contrary, that is, $\min\{f(\x_\utime) - f(\x_0), f(\modify{\x}_\utime) - f(\modify{\x}_0)\} > - \ufun$.
Lemma \ref{lem:improveorlocalize_gd} implies localization of both sequences around $\tilde{\x}$; that is, for any $t\le \utime$:
\begin{align}
\max\{\norm{\x_t - \tilde{\x}}, \norm{\modify{\x}_t - \tilde{\x}}\} 
\le& \max\{\norm{\x_t - \x_0}, \norm{\modify{\x}_t - \modify{\x}_0}\}  +\max\{\norm{\x_0 - \tilde{\x}}, \norm{\modify{\x}_0 - \tilde{\x}}\} \nn \\
\le& \sqrt{2\eta \utime\ufun}+ \eta r \le \uspace,  \label{eq:localization_gd}
\end{align}
where the last step is due to our choice of $\eta, r, \utime, \ufun, \uspace$, as in Eq.~\eqref{eq:para_gd}, and $\ell/\sqrt{\rho\epsilon} \ge 1$.\footnote{We note that when $\ell/\sqrt{\rho\epsilon} < 1$, $\epsilon$-second-order stationary points are equivalent to $\epsilon$-first-order stationary points due to the function $f$ being $\ell$-gradient Lipschitz. In this case, the problem of finding $\epsilon$-second-order stationary points becomes straightforward.}
On the other hand, we can write the update equation for the difference $\dif{\x}_t \defeq \x_t - \modify{\x}_t$ as:
\begin{align*}
\dif{\x}_{t+1} =& \dif{\x}_t -\eta [\grad f(\x_t) - \grad f(\modify{\x}_t)]
= (\I - \eta\H)\dif{\x}_t - \eta \Delta_t \dif{\x}_t \\
=& \underbrace{(\I - \eta\H)^{t+1}\dif{\x}_0}_{\p(t+1)} - \underbrace{\eta\sum_{\tau = 0}^t(\I - \eta\H)^{t-\tau} \Delta_\tau \dif{\x}_\tau}_{\q(t+1)},
\end{align*}
where $\H = \hess f(\tilde{\x})$ and $\Delta_t = \int_{0}^1 [\hess f(\modify{\x}_t + \theta (\x_t -\modify{\x}_t) - \H]\mathrm{d} \theta $. 
We note that $\p(t)$ arises from the initial difference $\dif{x}_0$, and $\q(t)$ is an error term which arises from the fact that the function $f$ is not quadratic. We now use induction to show that the error term $\q(t)$ is always small compared to the leading term $\p(t)$. That is, we wish to show:
\begin{equation*}
\norm{\q(t)} \le \norm{\p(t)} /2,  \qquad t \in [\utime].
\end{equation*}
The claim is true for the base case $t=0$ as $\norm{\q(0)} = 0 \le \norm{\dif{\x}_0}/2 = \norm{\p(0)}/2$. Now suppose the induction claim is true up to $t$. Denote $\lambda_{\min}(\hess f(\x_0)) = -\gamma$. Note that $\dif{\x}_0$ lies in the direction of the minimum eigenvector of $\hess f(\x_0)$. Thus for any $\tau \le t$, we have:
\begin{align*}
\norm{\dif{\x}_\tau}  \le \norm{\p(\tau)} + \norm{\q(\tau)} \le 2\norm{\p(\tau)}
= 2\norm{(\I - \eta\H)^{\tau}\dif{\x}_0} = 2(1+\eta\gamma)^\tau \eta r_0.
\end{align*}
By the Hessian Lipschitz property, we have $\norm{\Delta_t} \le \rho \max\{\norm{\x_t - \tilde{\x}}, \norm{\modify{\x}_t - \tilde{\x}} \} \le \rho\uspace$,
therefore:
\begin{align*}
\norm{\q(t+1)} & = \norm{\eta\sum_{\tau = 0}^{t}(\I - \eta\H)^{t-\tau} \Delta_\tau \dif{\x}_\tau }
\le \eta\rho\uspace \sum_{\tau = 0}^{t} \norm{(\I - \eta\H)^{t-\tau}}\norm{\dif{\x}_\tau} \\
&\le 2\eta\rho\uspace \sum_{\tau = 0}^{t} (1+\eta\gamma)^t \eta r_0
\le 2\eta\rho\uspace \utime (1+\eta\gamma)^t \eta r_0
\le 2\eta\rho\uspace \utime \norm{\p(t+1)},
\end{align*}
where the second-to-last inequality uses $t+1 \le \utime$. By our choice of the hyperparameter in Eq.~\eqref{eq:para_gd}, we have $2\eta\rho\uspace \utime \le 1/2$, which finishes the inductive proof. 

Finally, the induction claim implies:
\begin{align*}
\max\{\norm{\x_\utime - \x_0}, \norm{\modify{\x}_\utime - \x_0}\} \ge& \frac{1}{2}\norm{\dif{\x}(\utime)} \ge \frac{1}{2}[\norm{\p(\utime)}
-\norm{\q(\utime)}] \ge \frac{1}{4}[\norm{\p(\utime)} \\
=& \frac{(1+\eta\gamma)^{\utime} \eta r_0}{4}
\overset{(1)}{\ge} 2^{\logt-2} \eta r_0 > \uspace,
\end{align*}
where step (1) uses the fact $(1+x)^{1/x} \ge 2$ for any $x \in (0, 1]$. This contradicts the localization property of Eq.~\eqref{eq:localization_gd}, which finishes the proof.
\end{proof}



\begin{figure}[t]
\centering
\begin{minipage}{.45\textwidth}
  \centering
  \includegraphics[trim={2cm 2cm 2cm 0cm}, width=\textwidth]{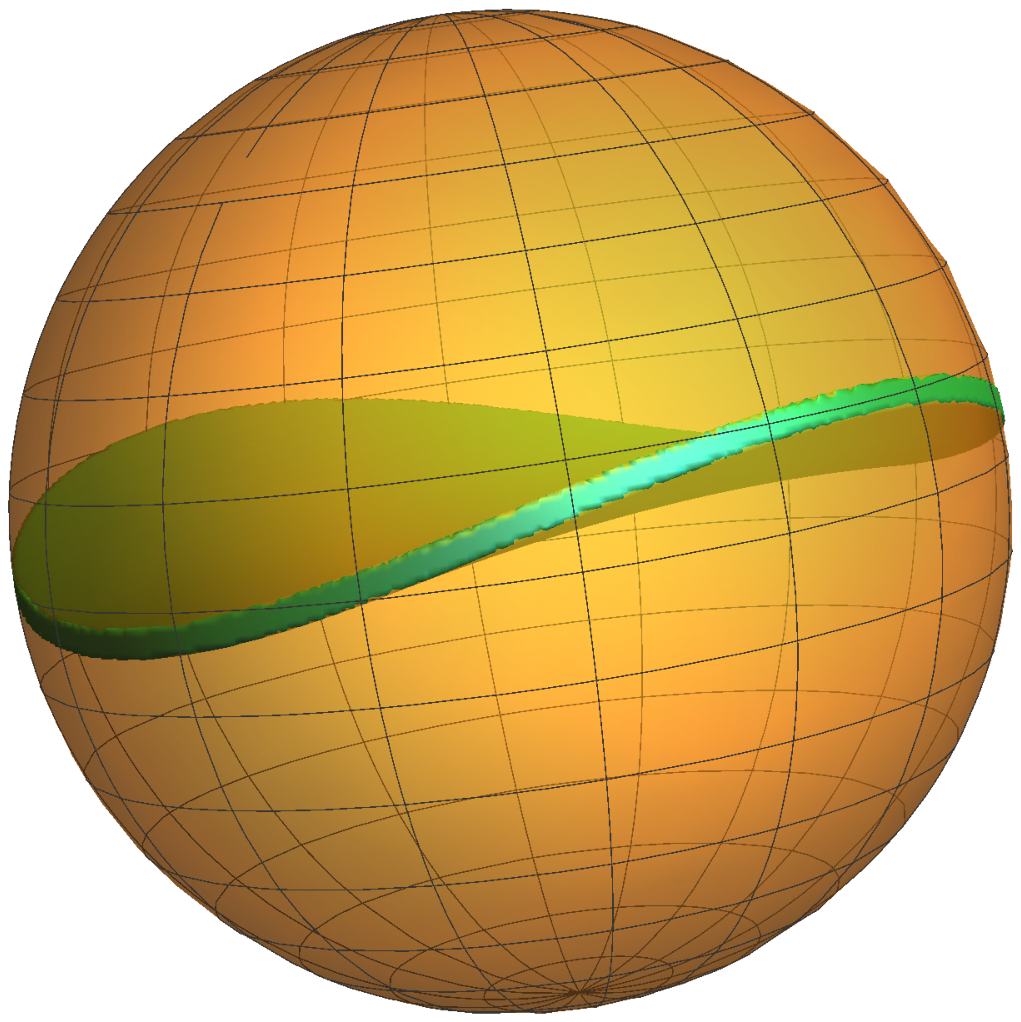}
  \captionof{figure}{Pertubation ball in 3D and ``thin pancake'' shape stuck region}
  \label{fig:band}
\end{minipage}%
\begin{minipage}{.05\textwidth}
~
\end{minipage}
\begin{minipage}{.45\textwidth}
  \centering
  \includegraphics[width=0.9\textwidth]{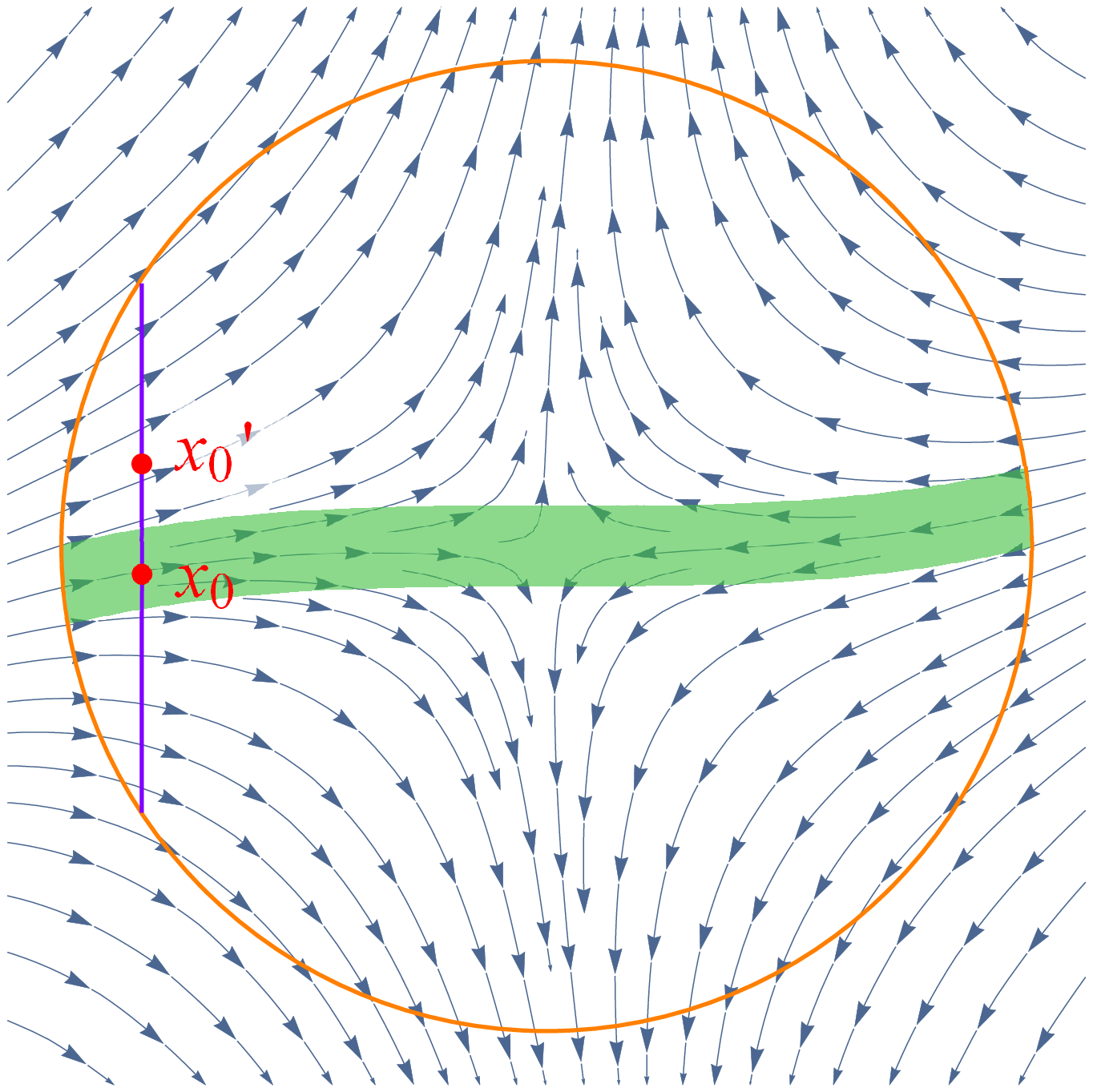}
  \captionof{figure}{Pertubation ball in 2D and ``narrow band'' stuck region under gradient flow}
  \label{fig:flow}
\end{minipage}
\end{figure}

Equipped with Lemma \ref{lem:improveorlocalize_gd} and Lemma \ref{lem:coupleseq_gd}, we are ready to prove Lemma \ref{lem:escapesaddle_gd}.

\begin{proof}[Proof of Lemma \ref{lem:escapesaddle_gd}]
Recall $\x_0 \sim \text{Uniform}(B_{\tilde{\x}}(\eta r))$. We refer to $B_{\tilde{\x}}(\eta r)$ the \emph{perturbation ball}, and define the \emph{stuck region} within the perturbation ball to be the set of points starting from which GD requires more than $\utime$ steps to escape:
\begin{equation*}
\cXs \defeq \{ \x \in B_{\tilde{\x}}(\eta r) ~|~ \{\x_t\} \text{~is GD sequence with~} \x_0 = \x, \text{and~} f(\x_\utime) - f(\x_0) > -\ufun \}.
\end{equation*}
See Figure \ref{fig:band} and Figure \ref{fig:flow} for illustrations. Although the shape of the stuck region can be very complicated, according to Lemma \ref{lem:coupleseq_gd} we know that the width of $\cXs$ along the $\e_1$ direction is at most $\eta \omega$. That is, $\text{Vol}(\cXs) \le \text{Vol}(\ball_0^{d-1}(\eta r)) \eta \omega$. Therefore:
\begin{align*}
\Pr(\x_0 \in \cXs) = \frac{\text{Vol}(\cXs)}{\text{Vol}(\ball^{d}_{\tilde{\x}}(\eta r))}
\le \frac{\eta \omega \times \text{Vol}(\ball^{d-1}_0(\eta r))}{\text{Vol} (\ball^{d}_0(\eta r))}
= \frac{\omega}{r\sqrt{\pi}}\frac{\Gamma(\frac{d}{2}+1)}{\Gamma(\frac{d}{2}+\frac{1}{2})}
\le \frac{\omega}{r} \cdot \sqrt{\frac{d}{\pi}} \le  \frac{\ell  \sqrt{d}}{\sqrt{\rho\epsilon}} \cdot \logt^2 2^{8-\logt}. 
\end{align*}
On the event $\{\x_0 \not \in \cXs\}$, due to our parameter choice in Eq.~\eqref{eq:para_gd}, we have:
\begin{equation*}
f(\x_\utime) - f(\tilde{\x}) = [f(\x_\utime) - f(\x_0)] + [f(\x_0)- f(\tilde{\x})]
\le -\ufun + \epsilon \eta r + \frac{\ell \eta^2 r^2}{2} \le -\ufun/2.
\end{equation*}
This finishes the proof.
\end{proof}


With Lemma \ref{lem:descent_gd} and Lemma \ref{lem:escapesaddle_gd} in hand, it is not hard to establish Theorem \ref{thm:main_gdez}.

\begin{proof}[Proof of Theorem \ref{thm:main_gdez}]
First, we set the total number of iterations $T$ to be:
\begin{equation*}
T = 8\max\left\{ \frac{(f(x_0) - f^\star)\utime}{\ufun},  \frac{(f(x_0) - f^\star)}{\eta\epsilon^2}\right\}
= O\left( \frac{\ell(f(x_0) - f^\star)}{\epsilon^2}\cdot \logt^4 \right).
\end{equation*}
Next, we choose $\logt = c \cdot \log( \frac{d \ell \Delta_f}{\rho\epsilon\delta})$, with a large enough absolute constant $c$ such that:
\begin{equation*}
(T  \ell  \sqrt{d} /\sqrt{\rho\epsilon} ) \cdot\logt^2  2^{8-\logt} \le \delta.
\end{equation*}
We then argue that, with probability $1-\delta$, Algorithm \ref{algo:PGDV} will add a perturbation at most $T/(4\utime)$ times.
This is because if otherwise, we can appeal to Lemma \ref{lem:escapesaddle_gd} every time we add a perturbation, and conclude:
\begin{equation*}
f(\x_T) \le f(\x_0) - T \ufun/(4\utime) < f^\star,
\end{equation*}
which can not happen. Finally, excluding those iterations that are within $\utime$ steps after adding perturbations, we still have $3T/4$ steps left.
They are either large gradient steps, $\norm{\grad f(\x_t)} \ge \epsilon$, or $\epsilon$-second order stationary points.
Within them, we know that the number of large gradient steps cannot be more than $T/4$. This is true because if otherwise, by Lemma \ref{lem:descent_gd}:
\begin{equation*}
f(\x_T) \le f(\x_0) - T \eta\epsilon^2 /4 < f^\star,
\end{equation*}
which again cannot happen. Therefore, we conclude that at least $T/2$ of the iterates must be $\epsilon-$second order stationary points.
\end{proof}


\section{Conclusions}\label{sec:conc}
We have shown that simple perturbed versions of GD and SGD escape saddle points and find second-order stationary points in essentially the same time that classical GD and SGD take to find first-order stationary points. The overheads are only logarithmic factors in dimensionality in both the non-stochastic setting and the stochastic setting with Lipschitz stochastic gradient. In the general stochastic setting, the overhead is a linear factor in dimension.

Combined with previous literature that shows that all second-order stationary points are global optima for a broad class of nonconvex optimization problems in machine learning and signal processing, our results directly provide efficient guarantees for solving those nonconvex problem via simple local search approaches. We now discuss several possible future directions, and further connections to other fields.


%
%

\paragraph{Optimal rates for finding second-order stationary points.}
\citet{carmon2017lower1} have presented lower bounds that imply that GD achieves the optimal rate for finding stationary points for gradient Lipschitz functions.  In our setting, we additionally assume that the Hessian is Lipschitz.  This implies that GD is no longer necessarily an optimal algorithm. While our results show that variants of GD are efficient in this setting, one would additionally like to know whether they are close to optimality.

Focusing on first-order algorithms, the best known gradient query complexity for finding second-order stationary points of functions with Lipschitz gradient and Lipschitz Hessian is $\tlO(\epsilon^{-1.75})$~\citep{carmon2016accelerated, agarwal2017finding, jin2017accelerated}, while the existing lower bound is $\Omega(\epsilon^{-12/7})$ by \citet{carmon2017lower2}. Note that this lower bound is restricted to deterministic algorithms, and thus does not apply to most existing algorithms for escaping saddle points as they are all randomized algorithms. For the stochastic setting with Lipschitz stochastic gradients, the best known query complexity is $\tlO(\epsilon^{-3})$~\citep{fang2018spider,zhou2019stochastic}, while the lower bound remains open. See Appendix \ref{app:table} for further discussion of the literature.

\paragraph{Escaping high-order saddle points.}
The current paper focuses on escaping strict saddle points and finding second-order stationary points. More generally, one can define $n$th-order stationary points as points that satisfies the KKT necessary conditions for being local minima up to $n$th-order derivatives. It becomes more challenging to find $n$th-order stationary points as $n$ increases, since it is necessary to escape higher-order saddle points. In terms of worst-case efficiency, \citet{nesterov2000squared} rules out the possibility of efficient algorithms for finding $n$th-order stationary points for all $n\ge 4$, showing that the problem is NP-hard. \citet{anandkumar2016efficient} present a third-order algorithm that finds third-order stationary points in polynomial time. It remains open whether simple variants of GD can also find third-order stationary points efficiently. It is unlikely that the overhead will still be small or only logarithmic in this case, but it is not clear what to expect for the overhead. A related question is to identify applications where third-order stationarity is needed, in addition to second-order stationarity, to achieve global optimality.

\paragraph{Connection to gradient Langevin dynamics.}
The Bayesian counterpart of SGD is the Langevin Monte Carlo (LMC) algorithm \citep{roberts1996exponential}, which performs the following iterative update:
\begin{equation*}
\x_{t+1} = \x_t - \eta (\grad f(\x_t) + \sqrt{2/(\eta\beta)} \w_t), \quad \text{where} \quad \w_t \sim \mathcal{N}(\zero, \I).
\end{equation*}
Here $\beta$ is known as the inverse temperature. When the step size $\eta$ goes to zero, the distribution of the LMC iterates is known to converge to the stationary distribution $\mu(\x) \propto e^{-\beta f(\x)}$ \citep{roberts1996exponential}.

While the LMC algorithm is superficially very close to stochastic gradient descent, the goals for the two algorithms are quite different.
\begin{itemize}
	\item \textbf{Convergence}: While the focus of the optimization literature is to find stationary points, the goal of the LMC algorithm is to converge to a stationary distribution (i.e., to mix rapidly).
	\item \textbf{Scaling of Noise}: The scaling of the stochasticity in SGD---and in particular the size of the perturbation that we consider in the current paper---is much smaller than the scaling considered in the LMC literature. Running our algorithm is equivalent to running LMC with temperature $\beta^{-1} \propto d^{-1}$. In this low-temperature or small-noise regime, the algorithm can no longer mix efficiently for smooth nonconvex functions, as it requires $\Omega(e^{d})$ steps in the worst case \citep{bovier2004metastability}. However, with this small amount of noise, the algorithm can still perform local search efficiently, and can find a second-order stationary point in a small number of iterations, as shown in Theorem \ref{thm:main_GD}.
\end{itemize}


Recent work of \citet{zhang2017hitting} studied the time that LMC takes to hit a second-order stationary point as a criterion for convergence, instead of the traditional mixing time to a stationary distribution.  In this analysis, the runtime is no longer exponential, but it is still polynomially dependent on dimension $d$ with large degree.

\paragraph{On the necessity of adding perturbations.} 
We have shown that adding perturbations to the iterations of GD or SGD allows these algorithms to escape saddle points efficiently. As an alternative, one can also simply run GD with random initialization, and try to escape saddle points using only the randomness due to the initialization. Although this alternative algorithm exhibits asymptotic convergence \citep{lee2016gradient}, it does not yield efficient convergence in general. \citet{du2017gradient} shows that even with fairly natural random initialization schemes and non-pathological functions, randomly initialized GD can be significantly slowed by saddle points, taking exponential time to escape them. 

\section*{Acknowledgements}
We thank Tongyang Li and Quanquan Gu for valuable 
discussions. This work was supported in part by the Mathematical Data Science program of the
Office of Naval Research under grant number N00014-18-1-2764. Rong Ge also acknowledges the funding support from NSF CCF-1704656, NSF CCF-1845171 (CAREER), Sloan Fellowship and Google Faculty Research Award.

\bibliographystyle{plainnat}
\bibliography{saddle}

\newpage
\appendix

\renewcommand{\arraystretch}{1.5}
\begin{table*}[t]
    \centering
    \begin{tabular}
    {|>{\centering\arraybackslash}m{2.6in} |>{\centering\arraybackslash}m{1.2in}|>{\centering\arraybackslash}m{0.9in}|}
        \hline
        \textbf{Algorithm} & \textbf{Iterations} & \textbf{Simplicity} \\       
        \hline
        Noisy GD \citep{ge2015escaping}
        &  $d^{4}\text{poly}(\epsilon^{-1})$  & \multirow{4}{0.9in}{\centering single-loop} \\
        \hhline{|--~|}
        Normalized GD \citep{levy2016power}
        & $\mathcal{O}(d^3\cdot\poly(\epsilon^{-1}))$ & \\
        \hhline{|--~|}
        \cellcolor{light-gray} \textbf{PGD (conference verison of this work) \citep{jin2017escape}} & \cellcolor{light-gray} $\mathcal{\tilde{O}}(\epsilon^{-2})$ &  \\
        \hhline{|--~|}
        ${}^{\dagger}$Perturbed AGD \citep{jin2017accelerated}
        &  $ \mathcal{\tilde{O}}(\epsilon^{-1.75})$ & \\
        \hhline{|===|}      
        FastCubic \citep{agarwal2017finding}
        &  $\mathcal{\tilde{O}}(\epsilon^{-1.75})$  & \multirow{3}{0.9in}{\centering double-loop} \\
        \hhline{|--~|}
        \citet{carmon2016accelerated}
        &  $ \mathcal{\tilde{O}}(\epsilon^{-1.75})$  & \\
        \hhline{|--~|}
        \citet{carmon2016gradient}
        &  $ \mathcal{\tilde{O}}(\epsilon^{-2})$  & \\
        \hline
    \end{tabular}
    \caption{A summary of related work on first-order algorithms to find second-order stationary points in \emph{non-stochastic} settings. This table only highlights the dependencies on $d$ and $\epsilon$.${}^\dagger$ denotes followup work on the conference version of this paper \citep{jin2017escape}.} 
    \label{tab:comparison_nonstochastic}
\end{table*}

\renewcommand{\arraystretch}{1.5}
\begin{table*}[h!]
    \centering
    \begin{tabular}
    {|>{\centering\arraybackslash}m{2.6in} |>{\centering\arraybackslash}m{1.0in} |>{\centering\arraybackslash}m{1.0in}|>{\centering\arraybackslash}m{0.9in}|}
        \hline
        \textbf{Algorithm} & \textbf{Iterations} (\emph{with} Assumption \ref{assump:SG_Lip})  & \textbf{Iterations} (\emph{no} Assumption \ref{assump:SG_Lip}) & \textbf{Simplicity} \\       
        \hline
        Noisy GD \citep{ge2015escaping}
        &  $d^{4}\text{poly}(\epsilon^{-1})$  & $d^{4}\text{poly}(\epsilon^{-1})$ & \multirow{4}{0.9in}{\centering single-loop} \\
        \hhline{|---~|}
        CNC-SGD \citep{daneshmand2018escaping}
        & $ \mathcal{\tilde{O}}(d^{4}\epsilon^{-5})$ & $ \mathcal{\tilde{O}}(d^{4}\epsilon^{-5})$ & \\
        \hhline{|---~|}
        \cellcolor{light-gray} \textbf{PSGD (this work)} & \cellcolor{light-gray} $\mathcal{\tilde{O}}(\epsilon^{-4})$ & \cellcolor{light-gray} $\mathcal{\tilde{O}}(d\epsilon^{-4})$ &  \\
        \hhline{|---~|}
        ${}^{*}$SGD with averaging \citep{fang2019}
        &  $ \mathcal{\tilde{O}}(\epsilon^{-3.5})$ & $\times$ & \\
        \hhline{|====|}      
        Natasha 2 \citep{allen2018natasha}
        &  $\mathcal{\tilde{O}}(\epsilon^{-3.5})$ & $\times$ & \multirow{4}{0.9in}{\centering double-loop} \\
        \hhline{|---~|}
        Stochastic Cubic \citep{tripuraneni2018stochastic}
        &  $ \mathcal{\tilde{O}}(\epsilon^{-3.5})$ & $\times$ & \\
        \hhline{|---~|}
        SPIDER \citep{fang2018spider}
        &  $ \mathcal{\tilde{O}}(\epsilon^{-3})$ & $\times$ & \\
        \hhline{|---~|}
        SRVRC \citep{zhou2019stochastic}
        &  $ \mathcal{\tilde{O}}(\epsilon^{-3})$ & $\times$ & \\
        \hline
    \end{tabular}
    \caption{A summary of related work on first-order algorithms to find second-order stationary points in the \emph{stochastic} setting. This table only highlights the dependencies on $d$ and $\epsilon$. ${}^*$ denotes independent work.}
    \label{tab:comparison_stochastic}
\end{table*}

\section{Tables of Related Work}

\label{app:table}

In Table \ref{tab:comparison_nonstochastic} and Table \ref{tab:comparison_stochastic}, we present a detailed comparison of our results with other related work in both non-stochastic and stochastic settings. See Section \ref{sec:related} for the full text descriptions. We note that our algorithms are simple variants of standard GD and SGD, which are the simplest among all the algorithms listed in the table.


\section{Proofs for Stochastic Setting}\label{app:proof}

In this section, we provide proofs for our main results---Theorem \ref{thm:main_SGD} and Theorem \ref{thm:main_SGD_minibatch}. Theorem \ref{thm:main_GD} can be proved as a special case of Theorem \ref{thm:main_SGD} by taking $\sigma=0$.

\subsection{Notation}
Recall the update equation of Algorithm \ref{algo:PSGD} is $\x_{t+1} \leftarrow \x_t - \eta (\g (\x_t;\theta_t) + \xi_t)$, where $\xi_t  \sim \cN(\zero, (r^2/d) \I)$. Throughout this section, we denote $\zeta_t \defeq  \g(\x_t; \theta_t) - \grad f(\x_t) $, as the noise part within the stochastic gradient.
For simplicity, we also denote $\tilde{\zeta}_t \defeq \zeta_t + \xi_t$, which is the summation of noise from the stochastic gradient and the injected perturbation, and $\tilde{\sigma}^2 \defeq \sigma^2 + r^2 $.
Then the update equation can be rewritten as $\x_{t+1} \leftarrow \x_t - \eta (\grad f(\x_t) + \tilde{\zeta}_t)$.
We also denote $\F_t = \sigma(\zeta_0, \xi_0, \ldots, \zeta_t, \xi_t)$ as the corresponding filtration up to time step $t$. We choose parameters in Algorithm \ref{algo:PSGD} as follows:
\begin{equation}\label{eq:para_proof}
\eta = \frac{1}{ \logt^9  \cdot \ell \Ns}, \qquad
r =  \logt \cdot \epsilon \sqrt{\Ns}, \qquad \utime \defeq \frac{\logt}{\eta \sqrt{\rho\epsilon}}, 
\qquad \ufun \defeq \frac{1}{\logt^5}\sqrt{\frac{\epsilon^3}{\rho}}, 
\qquad \uspace \defeq \frac{2}{\logt^2}\sqrt{\frac{\epsilon}{\rho}},
\end{equation}
where $\Ns$ and the log factor $\logt$ are defined as:
\begin{equation*}
\Ns = 1+ \min\left\{\frac{\sigma^2}{\epsilon^2} + \frac{\tilde{\ell}^2}{\ell\sqrt{\rho\epsilon}}, ~\frac{\sigma^2d}{\epsilon^2}\right\},
\qquad \logt = \mu\cdot \log \left( \frac{ d \ell \Delta_f \Ns}{\rho\epsilon\delta}\right).
\end{equation*}
Here, $\mu$ is a sufficiently large absolute constant to be determined later. Note also that throughout this section $c$ denotes an absolute constant that does not depend on the choice of $\mu$.
Its value may change from line to line.

\subsection{Descent lemma}

We first prove that the change in the function value can be always decomposed into the decrease due to the magnitudes of gradients and the possible increase due to randomness in both the stochastic gradients and the perturbations.
\begin{lemma}[Descent Lemma]\label{lem:descent}
 Under Assumption \ref{assump:GD}, \ref{assump:SG}, there exists an absolute constant $c$ such that, for any fixed $t, t_0, \logt>0$, if $\eta \le 1/\ell$, then with at least $1-4e^{-\logt}$ probability, the sequence PSGD$(\eta, r)$ (Algorithm \ref{algo:PSGD}) satisfies:
(denote $\tilde{\sigma}^2 = \sigma^2+r^2$)
\begin{equation*}
f(\x_{t_0+t}) - f(\x_{t_0}) \le -\frac{\eta}{8}\sum_{i=0}^{t-1} \norm{\grad f(\x_{t_{0}+i})}^2 + c\cdot\eta \tilde{\sigma}^2 (\eta \ell t + \logt).
\end{equation*}
\end{lemma}

\begin{proof}
Since Algorithm \ref{algo:PSGD} is Markovian, the operations in each iteration do not depend on time step $t$. Thus, it suffices to prove Lemma \ref{lem:descent} for special case $t_0 = 0$. 
Recall the update equation:
\begin{equation*}
\x_{t+1} \leftarrow \x_t - \eta ( \grad f(\x_t) + \tilde{\zeta}_t),
\end{equation*}
where $\tilde{\zeta}_t = \zeta_t + \xi_t$. By assumption, we know $\zeta_t | \F_{t-1}$ is zero-mean $\nSG(\sigma)$. Also $\xi_t | \F_{t-1}$ comes from $\cN(\zero, (r^2/d) \I)$, and thus by Lemma \ref{lem:examplesGnorm} it is zero-mean $\nSG(c \cdot r)$ for some absolute constant $c$. By a Taylor expansion and the assumptions of an $\ell$-gradient Lipschitz and $\eta \le 1/\ell$, we have:
\begin{align*}
f(\x_{t+1}) 
\le & f(\x_t) +  \la \grad f(\x_t), \x_{t+1} - \x_t\ra + \frac{\ell}{2}\norm{\x_{t+1} -\x_t}^2\\
\le & f(\x_t) - \eta  \la \grad f(\x_t), \grad f(\x_t) + \tilde{\zeta}_t \ra + \frac{\eta^2\ell}{2} \left[\frac{3}{2}\norm{\grad f(\x_t)}^2 + 3\norm{\tilde{\zeta}_t}^2\right] \\
\le & f(\x_t) - \frac{\eta}{4}\norm{\grad f(\x_t)}^2 
-\eta \la \grad f(\x_t), \tilde{\zeta}_t \ra +\frac{3}{2}\eta^2\ell \norm{\tilde{\zeta}_t}^2.
\end{align*}
Summing over this inequality, we have the following:
\begin{equation}
f(\x_{t}) - f(\x_0)
\le  -\frac{\eta}{4}\sum_{i=0}^{t-1} \norm{\grad f(\x_i)}^2 - \eta \sum_{i=0}^{t-1}  \la \grad f(\x_i), \tilde{\zeta}_i \ra
+ \frac{3}{2} \eta^2\ell \sum_{i=0}^{t-1}  \norm{\tilde{\zeta}_i}^2. \label{eq:decent_decomp}
\end{equation}
For the second term on the right-hand side, applying Lemma \ref{lem:concen_inner}, there exists an absolute constant $c$, such that, with probability $1-2e^{-\logt}$:
\begin{equation*}
- \eta \sum_{i=0}^{t-1}  \la \grad f(\x_i), \tilde{\zeta}_i\ra \le \frac{\eta}{8}\sum_{i=0}^{t-1} \norm{\grad f(\x_i)}^2
+ c \eta\tilde{\sigma}^2 \logt.
\end{equation*}
For the third term on the right-hand side of Eq.~\eqref{eq:decent_decomp}, applying Lemma \ref{lem:concen_square}, with probability $1-2e^{-\logt}$:
\begin{equation*}
\frac{3}{2} \eta^2\ell \sum_{i=0}^{t-1}  \norm{\tilde{\zeta}_i}^2
\le 3\eta^2\ell\sum_{i=0}^{t-1}  (\norm{\zeta_i}^2 + \norm{\xi_i}^2)
\le c \eta^2\ell \tilde{\sigma}^2 (t + \logt).
\end{equation*}
Substituting these inequalities into Eq.~\eqref{eq:decent_decomp}, and noting the fact $\eta \le 1/\ell$, we have with probability $1-4e^{-\logt}$:
\begin{equation*}
f(\x_{t}) - f(\x_0) \le -\frac{\eta}{8}\sum_{i=0}^{t-1} \norm{\grad f(\x_i)}^2
+ c \eta\tilde{\sigma}^2 (\eta \ell t + \logt).
\end{equation*}
This finishes the proof. 
\end{proof}

The descent lemma allows us to show the following ``Improve or Localize'' phenomenon for perturbed SGD. That is, with high probability over a small number of iterations, either the function value decreases significantly, or the iterates stay within a small local region.

\begin{lemma}[Improve or Localize]\label{lem:locality_SGD}
Under the same setting as in Lemma \ref{lem:descent}, with probability at least $1-8dt \cdot e^{-\logt}$, the sequence PSGD$(\eta, r)$ (Algorithm \ref{algo:PSGD}) satisfies:
\begin{equation*}
\forall \tau \le t: \norm{\x_{t_0+\tau} - \x_{t_0}}^2
\le c\eta t \cdot [ f(\x_{t_0}) - f(\x_{t_0 + \tau})  + \eta \tilde{\sigma}^2  (\eta\ell t + \logt)].
\end{equation*}
\end{lemma}

\begin{proof}
By a similar argument as in the proof of Lemma \ref{lem:descent}, it suffices to prove Lemma \ref{lem:locality_SGD} in the special case $t_0 = 0$.
According to Lemma \ref{lem:descent}, with probability $1-4e^{-\logt}$, for some absolute constant $c$:
\begin{equation*}
\sum_{i=0}^{t-1} \norm{\grad f(\x_i)}^2 \le  \frac{8}{\eta}[f(\x_0)-f(\x_{t})] + c\tilde{\sigma}^2 (\eta \ell t + \logt ).
\end{equation*}

Therefore, for any fixed $\tau \le t$, with probability $1-8d\cdot e^{-\logt}$,
\begin{align*}
\norm{\x_{\tau} - \x_0}^2 =& \eta^2\norm{\sum_{i=0}^{\tau-1} (\grad f(\x_i) + \tilde{\zeta}_i)}^2 \le 2\eta^2 \left[\norm{\sum_{i=0}^{\tau-1} \grad f(\x_i) }^2 + \norm{\sum_{i=0}^{\tau-1} \tilde{\zeta}_i}^2 \right] \\
\overset{(1)}{\le}& 2\eta^2 t\sum_{i=0}^{\tau-1} \norm{\grad f(\x_i) }^2 + c \eta^2 \tilde{\sigma}^2 t \logt
\le  2\eta^2 t \sum_{i=0}^{t-1} \norm{\grad f(\x_i) }^2 + c \eta^2 \tilde{\sigma}^2 t \logt \\
\le& c\eta t [ f(\x_0) - f(\x_t)  + \eta \tilde{\sigma}^2 (\eta\ell t + \logt)].
\end{align*}
where in step (1) we use the Cauchy-Schwarz inequality and Lemma \ref{lem:concen_sum}. Finally, applying a union bound for all $\tau \le t$, we finish the proof.
\end{proof}

\subsection{Escaping saddle points}
 Lemma \ref{lem:descent} shows that large gradients contribute to the fast decrease of the function value. In this subsection, we will show that starting in the vicinity of strict saddle points will also enable PSGD to decrease the function value rapidly. Concretely, this entire subsection will be devoted to proving the following lemma:
\begin{lemma}[Escaping Saddle Point]\label{lem:escape_saddle}
 Given Assumption \ref{assump:GD}, \ref{assump:SG}, there exists an absolute constant $c_{\max}$ such that, for any fixed $t_0>0, \logt> c_{\max} \log (\ell \sqrt{d/(\rho\epsilon)}) $, if $\eta, r, \ufun, \utime$ are chosen as in Eq.~\eqref{eq:para_proof}, and $\x_{t_0}$ satisfies $\norm{\grad f(\x_{t_0})} \le \epsilon$ and $\lambda_{\min}(\hess f(\x_{t_0})) \le -\sqrt{\rho\epsilon}$, then the sequence PSGD$(\eta, r)$ (Algorithm \ref{algo:PSGD}) satisfies:
\begin{align*}
&\Pr(f(\x_{t_0+\utime}) - f(\x_{t_0}) \le  0.1\ufun) \ge 1 - 4e^{-\logt}  \quad\quad \text{and}\\ 
&\Pr(f(\x_{t_0+\utime}) - f(\x_{t_0}) \le  -\ufun) \ge 1/3 - 5 d \utime^2 \cdot \log(\uspace\sqrt{d}/(\eta r)) e^{-\logt}.
\end{align*}
\end{lemma}

Since Algorithm \ref{algo:PSGD} is Markovian, the operations in each iterations do not depend on time step $t$. Thus, it suffices to prove Lemma \ref{lem:escape_saddle} for special case $t_0 = 0$. 
To prove this lemma, we first need to introduce the concept of a coupling sequence.

\paragraph{Notation:} Throughout this subsection, we let $\H \defeq \hess f(\x_0)$, $\e_1$ be the minimum eigendirection of $\H$, and $\gamma \defeq \lambda_{\min}(\H)$. We also let $\proj_{-1}$ be the projection onto the subspace complement of $\e_1$.

\begin{definition}[Coupling Sequence] \label{def:IndCouple}
Consider sequences $\{\x_i\}$ and $\{\modify{\x}_i\}$ that are obtained as  separate runs of the PSGD (algorithm \ref{algo:PSGD}), both starting from $\x_0$. 
They are coupled if both sequences share the same randomness $\proj_{-1} \xi_\tau$ and $\theta_\tau$, while in $\e_1$ direction we have $\e_1\trans \xi_\tau = -\e_1\trans \modify{\xi}_\tau$.
\end{definition}

The first thing we can show is that if the function values of both sequences do not exhibit a sufficient decrease, then both sequences are localized in a small ball around $\x_0$ within $\utime$ iterations.

\begin{lemma}[Localization]{\label{lem:localization}}
Under the notation of Lemma \ref{lem:escape_dynamic}, we have:
\begin{align*}
\Pr(\min\{f(\x_{\utime}) - f(\x_0), &f(\modify{\x}_{\utime}) - f(\x_0)\} \le -\ufun, \text{~~or~~} \\
&\forall t \le \utime:  \max\{\norm{\x_t - \x_0}^2, \norm{\modify{\x}_t - \x_0}^2\} \le \uspace^2 ) \ge 1-16d\utime\cdot e^{-\logt}.
\end{align*}
\end{lemma}
\begin{proof}
This lemma follows from applying Lemma \ref{lem:locality_SGD} on both sequences and using a union bound.
\end{proof}

The overall proof strategy for Lemma \ref{lem:escape_saddle} is to show that localization happens with a very small probability, 
thus at least one of the sequence must have sufficient descent. In order to prove this, we study the dynamics of the difference of the coupling sequence.
\begin{lemma}[Dynamics of the Coupling Sequence Difference] \label{lem:escape_dynamic}
Consider coupling sequences $\{\x_i\}$ and $\{\modify{\x}_i\}$
as in Definition \ref{def:IndCouple} and let $\dif{\x}_t \defeq \x_i - \modify{\x}_i$. Then
$\dif{\x}_{t} = -\qa(t)- \qb(t)-\qc(t)$, where:
\begin{equation*}
\qa(t) \defeq \eta\sum_{\tau = 0}^{t-1} (\I-\eta \H)^{t-1 - \tau} \Delta_{\tau} \dif{\x}_{\tau}, ~~ \qb(t) \defeq  \eta\sum_{\tau = 0}^{t-1} (\I-\eta \H)^{t-1-\tau}  \dif{\zeta}_{\tau}, ~~ \qc(t) \defeq  \eta\sum_{\tau = 0}^{t-1} (\I-\eta \H)^{t-1-\tau} \dif{\xi}_{\tau}.
\end{equation*}
Here $\Delta_t \defeq \int_0^1 \hess f(\psi \x_t + (1-\psi)\modify{\x_t}) \dd \psi - \H$, and 
$\dif{\zeta}_\tau \defeq \zeta_\tau - \modify{\zeta}_\tau$,  $\dif{\xi}_\tau \defeq \xi_\tau - \modify{\xi}_\tau$.
\end{lemma} 
\begin{proof}
Recall $\zeta_i = \g (\x_i; \theta_i) - \grad f(\x_i) $, thus, we have the update formula:
\begin{equation*}
\x_{t+1} = \x_t - \eta (\g(\x_t; \theta_t) + \xi_t) 
= \x_t - \eta (\grad f(\x_t) + \zeta_t + \xi_t).
\end{equation*}
Taking the difference between $\{\x_i\}$ and $\{\modify{\x}_i\}$:
\begin{align*}
\dif{\x}_{t+1} =& \x_{t+1} - \modify{\x}_{t+1}
=\dif{\x}_t - \eta(\grad f(\x_t) - \grad f(\modify{\x}_t)
+ \zeta_t - \modify{\zeta}_t + (\xi_t - \modify{\xi}_t))\\
=& \dif{\x}_t - \eta[(\H + \Delta_t)\dif{\x}_t + \dif{\zeta}_t + \dif{\xi}_t]
= (\I - \eta\H)\dif{\x}_t  -\eta [\Delta_t \dif{\x}_t + \dif{\zeta}_t + \e_1\e_1\trans\dif{\xi}_t]\\
= &-\eta \sum_{\tau = 0}^{}{t} (I-\eta \H)^{t - \tau}(\Delta_{\tau} \dif{\x}_{\tau} + \dif{\zeta}_{\tau}+ \dif{\xi}_{\tau})),
\end{align*}
where $\Delta_t \defeq \int_0^1 \hess f(\psi \x_t + (1-\psi)\modify{\x_t}) \dd \psi - \H$. This finishes the proof.
\end{proof}

At a high level, we will show that with constant probability, $\qc(t)$ is the dominating term which controls the behavior of the dynamics, and $\qa(t)$ and $\qb(t)$ will stay small compared to $\qc(t)$. To show this, we prove the following three lemmas.

\begin{lemma}\label{lem:escape_scalar}
Denote $\coef(t) \defeq \left[\sum_{\tau = 0}^{t-1} (1 + \eta \gamma)^{2(t-1 - \tau)} \right]^{\frac{1}{2}}$ and $\coefB(t) \defeq (1+\eta\gamma)^{t}/\sqrt{2\eta\gamma}$.
If $\eta \gamma \in [0, 1]$, then (1) $\coef(t) \le \coefB(t)$ for any $t \in \N$; and (2) $\coef(t) \ge \coefB(t)/\sqrt{3}$ for $t \ge \ln(2)/(\eta \gamma)$.
\end{lemma}
\begin{proof}
By summing a geometric sequence:
\begin{equation*}
\coef^2(t) \defeq \sum_{\tau = 0}^{t-1} (1 + \eta \gamma)^{2(t-1 - \tau)}  = \frac{(1 + \eta \gamma)^{2t} - 1}{2\eta\gamma + (\eta\gamma)^2}.
\end{equation*}
Thus, the claim that $\coef(t) \le \coefB(t)$ for any $t \in \N$ follows immediately. On the other hand, note that for $t \ge \ln(2)/(\eta \gamma)$, we have
$(1+\eta\gamma)^{2t} \ge 2^{2\ln2} \ge 2$, where the second claim follows by a short calculation.
\end{proof}

\begin{lemma}\label{lem:escape_perturb_term}
Under the notation of Lemma \ref{lem:escape_dynamic} and Lemma \ref{lem:escape_scalar}, letting $-\gamma\defeq \lambda_{\min}(\H)$, we have $\forall t>0$:
\begin{align*}
&\Pr(\norm{\qc(t)} \le  \frac{c\coefB(t)\eta r }{\sqrt{d}} \cdot  \sqrt{\logt} ) \ge 1-2 e^{-\logt} \\
&\Pr(\norm{\qc(\utime)} \ge   \frac{\coefB(\utime) \eta r}{10 \sqrt{d}} ) \ge \frac{2}{3}.
\end{align*}
\end{lemma}
\begin{proof}
 Note that $\dif{\xi}_\tau$ is one-dimensional Gaussian random variable with standard deviation $2r/\sqrt{d}$ along the $\e_1$ direction.
As an immediate result, $\eta\sum_{\tau = 0}^{t} (I-\eta \H)^{t - \tau} \dif{\xi}_{\tau}$ is also a one-dimensional Gaussian distribution since the summation of Gaussians is again Gaussian. Finally note that $\e_1$ is an eigendirection of $\H$ with corresponding eigenvalue $-\gamma$, and by Lemma \ref{lem:escape_scalar} we have $\coef(t) \le \coefB(t)$. Then, the first inequality immediately follows from the standard concentration of Gaussian measure, and the second inequality follows from the fact if $Z \sim \mathcal{N}(0, \sigma^2)$ then $\Pr(|Z| \le \lambda \sigma) \le 2\lambda/\sqrt{2\pi} \le \lambda$.
\end{proof}

\begin{lemma}\label{lem:escape_other_term}
There exists an absolute constant $c_{\max}$ such that, for any $\logt \ge c_{\max}$, under the notation of Lemma \ref{lem:escape_dynamic} and Lemma \ref{lem:escape_scalar}, and letting $-\gamma\defeq \lambda_{\min}(\H)$, we have:
\begin{align*}
\Pr(\min\{f(\x_{\utime}) - f(\x_0), &f(\modify{\x}_{\utime}) - f(\x_0)\} \le -\ufun, \text{~~or~~}\\
&\forall t \le \utime: \norm{\qa(t) + \qb(t)} \le \frac{  \coefB(t) \eta r}{20\sqrt{ d}}) \ge 1-10 d \utime^2 \cdot \log(\frac{\uspace\sqrt{d}}{\eta r}) e^{-\logt}.
\end{align*}
\end{lemma}

\begin{proof}
For simplicity we denote $\fE$ as the event $\{\forall \tau \le t:  \max\{\norm{\x_{\tau} - \x_0}^2, \norm{\modify{\x}_{\tau} - \x_0}^2\} \le \uspace^2\}$. We use induction to prove following claim for any $t\in [0, \utime]$:
\begin{align*}
\Pr(\fE ~\Rightarrow~ \forall \tau \le t: \norm{\qa(\tau) + \qb(\tau)} \le \frac{  \coefB(\tau) \eta r}{20\sqrt{ d}}) \ge 1-10 d \utime t \cdot \log(\frac{\uspace\sqrt{d}}{\eta r}) e^{-\logt}
\end{align*}
Then Lemma \ref{lem:escape_other_term} follows directly from combining Lemma \ref{lem:localization} and the induction claim.

Clearly for the base case $t=0$, the claim holds trivially, as $\qb(0) = \qa(0) = \mat{0}$. Suppose the claim holds for $t$, then by Lemma \ref{lem:escape_perturb_term}, with probability at least $1-2\utime e^{-\logt}$, we have for any $\tau \le t$: 
\begin{equation*}
\norm{\dif{\x}_{\tau}} 
\le \eta\norm{\qa(\tau) + \qb(\tau)} 
+ \eta\norm{\qc(\tau)}
\le \frac{c \coefB(\tau) \eta r}{\sqrt{ d}} \cdot \sqrt{\logt}.
\end{equation*}
Then, under the condition $\max\{\norm{\x_{\tau} - \x_0}^2, \norm{\modify{\x}_{\tau} - \x_0}^2\} \le \uspace^2$, 
by the Hessian Lipschitz property, we have $\norm{\Delta_{\tau}} = \norm{\int_0^1 \hess f(\psi \x_{\tau} + (1-\psi)\modify{\x_{\tau}}) \dd \psi - \H}
\le \rho \max\{\norm{\x_{\tau} - \x_0}, \norm{\modify{\x}_{\tau} - \x_0}\} \le \rho \uspace$. This gives bounds on $\qa(t+1)$ terms as:
\begin{equation*}
\norm{\qa(t+1)} \le \eta\sum_{\tau = 0}^{t} (1+\eta \gamma)^{t-\tau} \rho\uspace \norm{\dif{\x}_{\tau}}
\le \eta \rho \uspace \utime \frac{c \coefB(t) \eta r}{\sqrt{ d}} \le  \frac{ \coefB(t) \eta r}{40\sqrt{ d}},
\end{equation*}
where the last step is due to $\eta \rho \uspace \utime = 1/\logt$ by Eq.~\eqref{eq:para_proof}. By picking $\logt$ larger than the absolute constant $40 c$, we have $c\eta \rho \uspace \utime \le 1/40$.

Recall also that $\dif{\zeta}_{\tau}|\F_{\tau-1}$ is the summation of a $\nSG(\sigma)$ random vector and a $\nSG(c\cdot r)$ random vector.  By Lemma \ref{lem:concen_sum}, we know that with probability at least $1-4de^{-\logt}$:
\begin{equation*}
\norm{\qb(t+1)} \le c\coefB(t+1) \eta\sigma \sqrt{\logt}
\end{equation*}
On the other hand, when assumption \ref{assump:SG_Lip} is avaliable, we also have $\dif{\zeta}_{\tau}|\F_{\tau-1} \sim \nSG(\tilde{\ell}\norm{\dif{\x}_{\tau}})$, by applying Lemma \ref{lem:concen_sum_randB} with $B = \coef^2(t) \cdot \eta^2 \tilde{\ell}^2 \uspace^2; b = \coef^2(t)  \cdot \eta^2\tilde{\ell}^2 \cdot \eta^2 r^2 /d$, we know with probability at least $1-4d \cdot\log(\uspace \sqrt{d}/(\eta r))\cdot e^{-\logt}$:
\begin{equation} \label{eq:SGlipcoupling}
\norm{\qb(t+1)}
\le c \eta \tilde{\ell} \sqrt{\sum_{\tau = 0}^{t} (1+\eta \gamma)^{2(t- \tau)} \cdot
\max\{\norm{\dif{\x}_\tau}^2, \frac{\eta^2 r^2}{d}\} \logt } 
\le \eta \tilde{\ell}\sqrt{\utime} \cdot \frac{c \coefB(t) \eta r}{\sqrt{d}} \cdot  \sqrt{\logt}.
\end{equation}
Finally, combining both cases, and by our choice of step size $\eta, r$ as in Eq.~\eqref{eq:para_proof} with $\logt$ large enough:
\begin{equation*}
\norm{\qb(t+1)} \le c\frac{ \coefB(t) \eta r}{\sqrt{ d}}\cdot \min\{\eta\tilde{\ell}\sqrt{\utime\logt} , \frac{\sigma\sqrt{d \logt}}{r}\}
\le \frac{\coefB(t) r}{40\sqrt{d }}
\end{equation*}
and the induction follows by the triangle inequality and a union bound.
\end{proof}

We are ready to prove Lemma \ref{lem:escape_saddle}, which is the focus of this subsection.

\begin{proof}[Proof of Lemma \ref{lem:escape_saddle}]
We first prove the first claim $\Pr(f(\x_\utime) - f(\x_0) \le  0.1\ufun) \ge 1-4e^{-\logt}$.  Because of our choice of step size and Lemma \ref{lem:descent}, we have with probability $1-4e^{-\logt}$:
\begin{equation*}
f(\x_{\utime}) - f(\x_0) \le c\eta\tilde{\sigma}^2 (\eta \ell \utime + \logt) \le 0.1 \ufun,
\end{equation*}
where the last step is because our choice of parameters in Eq.~\eqref{eq:para_proof}implies $c\eta\tilde{\sigma}^2 (\eta \ell \utime + \logt)
\le 2c\ufun/\logt$ and we pick $\logt$ to be larger than an absolute constant $20c$.

For the second claim, $\Pr(f(\x_\utime) - f(\x_0) \le  -\ufun) \ge 1/3 - 5 d \utime^2 \cdot \log(\uspace\sqrt{d}/(\eta r)) e^{-\logt}$, we consider coupling sequences $\{\x_i\}$ and $\{\modify{\x}_i\}$ as defined in Definition \ref{def:IndCouple}. Given Lemma \ref{lem:escape_perturb_term} and Lemma \ref{lem:escape_other_term}, we know that with probability at least $2/3 - 10 d \utime^2 \cdot \log(\uspace\sqrt{d}/(\eta r)) e^{-\logt}$, if $\min\{f(\x_{\utime}) - f(\x_0), f(\modify{\x}_{\utime}) - f(\x_0)\} > -\ufun$---i.e., both sequences are stuck around the saddle point---we must have:
\begin{equation*}
\norm{\qc(\utime)} \ge   \frac{\coefB(\utime) \eta r}{10 \sqrt{d}}, 
\quad \norm{\qa(\utime) + \qb(\utime)} \le \frac{  \coefB(\utime) \eta r}{20\sqrt{ d}}.
\end{equation*}
By Lemma \ref{lem:escape_dynamic}, when $\logt \ge c \cdot \log (\ell \sqrt{d/(\rho\epsilon)})$ for a large absolute constant $c$, we have:
\begin{align*}
\max\{\norm{\x_\utime - \x_0}, \norm{\modify{\x}_\utime - \x_0}\} \ge &\frac{1}{2}\norm{\dif{\x}(\utime)}  
\ge  \frac{1}{2}[\norm{\qc(\utime)} -\norm{\qa(\utime) + \qb(\utime)}] \\
\ge&\frac{  \coefB(\utime) \eta r}{40\sqrt{ d}} 
= \frac{  (1+\eta \gamma)^\utime \eta r}{40\sqrt{2\eta\gamma d}} 
\le \frac{2^\logt \eta r}{80 \sqrt{\eta \ell d}},
> \uspace
\end{align*}
which contradicts with Lemma \ref{lem:localization}. Therefore, we can conclude that 
$\Pr(\min\{f(\x_{\utime}) - f(\x_0), f(\modify{\x}_{\utime}) - f(\x_0)\} \le -\ufun ) \ge 2/3 - 10 d \utime^2 \cdot \log(\uspace\sqrt{d}/(\eta r)) e^{-\logt}$.
We also know that the marginal distribution of $\x_{\utime}$ and $\modify{\x}_{\utime}$ is the same, thus they have same probability to escape the saddle point. That is:
\begin{align*}
\Pr(f(\x_\utime) - f(\x_0) \le  -\ufun) \ge & \frac{1}{2}\Pr(\min\{f(\x_{\utime}) - f(\x_0), f(\modify{\x}_{\utime}) - f(\x_0)\} \le -\ufun ) \\
\ge & 1/3 -5 d \utime^2 \cdot \log(\uspace\sqrt{d}/(\eta r)) e^{-\logt}.
\end{align*}
This finishes the proof.
\end{proof}

\subsection{Proof of Theorem \ref{thm:main_SGD}}

Lemma \ref{lem:descent} and Lemma \ref{lem:escape_saddle} describe the speed of decrease in the function values when either large gradients or strictly negative curvatures are present. Combining them gives the proof for our main theorem.

\begin{proof}[Proof of Theorem \ref{thm:main_SGD}]
First, we set the total number of iterations $T$ to be:
\begin{equation*}
T = 100\max\left\{ \frac{(f(x_0) - f^\star)\utime}{\ufun},  \frac{(f(x_0) - f^\star)}{\eta\epsilon^2}\right\}
= O\left( \frac{\ell(f(x_0) - f^\star)}{\epsilon^2}\cdot \Ns \cdot \logt^9 \right).
\end{equation*}
We will show that the following \textbf{two claims} hold simultaneously with probability $1-\delta$:
\begin{enumerate}
\item At most $T/4$ iterates have large gradient; i.e., $\norm{\grad f(\x_t)} \ge \epsilon$;
\item At most $T/4$ iterates are close to saddle points; i.e., $\norm{\grad f(\x_t)} \le \epsilon$ and $\lambda_{\min} (\hess f(\x_t)) \le -\sqrt{\rho \epsilon}$.
\end{enumerate}
Therefore, at least $T/2$ iterates are $\epsilon$-second order stationary point. We prove the two claims separately.

\paragraph{Claim 1.} Suppose that within $T$ steps, we have more than $T/4$ iterates for which gradient is large (i.e., $\norm{\grad f(\x_t)} \ge \epsilon$). Recall that by Lemma \ref{lem:descent} we have with probability $1-4e^{-\logt}$:
\begin{equation*}
f(\x_{T}) - f(\x_0) \le -\frac{\eta}{8}\sum_{i=0}^{T-1} \norm{\grad f(\x_i)}^2 + c\eta\tilde{\sigma}^2 (\eta \ell T + \logt) \\
\le - \eta \left[\frac{ T \epsilon^2}{32}  - \tilde{\sigma}^2 (\eta \ell T + \logt)\right].
\end{equation*}
Note that by our choice of $\eta, r, T$ and picking $\logt$ larger than some absolute constant, we have $ T \epsilon^2/32  - \tilde{\sigma}^2 (\eta \ell T + \logt) \ge  T \epsilon^2/64$, and thus 
$f(\x_T) \le f(x_0) - T\eta  \epsilon^2  / 64 < f^\star$ which is not possible.

\paragraph{Claim 2.} We first define the stopping times that allow us to invoke Lemma \ref{lem:escape_saddle}:
\begin{align*}
z_1 =& \inf\{\tau ~|~ \norm{\grad f(\x_{\tau})} \le \epsilon \text{~and~} \lambda_{\min} (f(\x_{\tau})) \le -\sqrt{\rho\epsilon}\}\\
z_i= &\inf\{\tau ~|~ \tau >z_{i-1} + \utime \text{~and~} \norm{\grad f(\x_{\tau})} \le \epsilon \text{~and~} \lambda_{\min} (f(\x_{\tau})) \le -\sqrt{\rho\epsilon}\}, \qquad \forall i>1.
\end{align*}
Clearly, $z_i$ is a stopping time, and it is the $i$th time in the sequence along which we can apply Lemma \ref{lem:escape_saddle}. We also let $M$ be the random variable $M = \max\{i|z_i + \utime \le T\}$. We can decompose the decrease $f(\x_T) - f(\x_0)$ as follows:
\begin{align*}
f(\x_{T}) - f(\x_0) =& \underbrace{\sum_{i=1}^{M} [f(\x_{z_i + \utime}) - f(\x_{z_i})]}_{T_1} \\
& + \underbrace{[f(\x_{T}) - f(\x_{z_M})] + [f(\x_{z_1}) - f(\x_0)] + \sum_{i=1}^{M-1}[f(\x_{z_{i+1}}) - f(\x_{z_i + \utime})] }_{T_2}.
\end{align*}
For the first term $T_1$, by Lemma \ref{lem:escape_saddle} and a supermartingale concentration inequality, for each fixed $m \le T$:
\begin{equation*}
\Pr\left(\sum_{i=1}^{m} [f(\x_{z_i + \utime}) - f(\x_{z_i})]\le -(0.9m - c\sqrt{m \cdot \logt})\ufun \right) \ge 1-5 d \utime^2 T \cdot \log(\uspace\sqrt{d}/(\eta r)) e^{-\logt}.
\end{equation*}
Since the random variable $M \le T/\utime \le T$, by a union bound, we know that with probability $1-5 d \utime^2 T^2 \cdot \log(\uspace\sqrt{d}/(\eta r)) e^{-\logt}$:
\begin{equation*}
T_1\le -(0.9M - c\sqrt{M \cdot \logt})\ufun. 
\end{equation*}
For the second term, by a union bound and Lemma \ref{lem:descent} for all $0\le t_1, t_2 \le T$, with probability $1 - 4T^2e^{-\logt}$:
\begin{equation*}
T_2 \le c\cdot\eta \tilde{\sigma}^2 (\eta \ell T + 2M\logt)
\end{equation*}
Therefore, if within $T$ steps we have more than $T/4$ saddle points, then $M \ge T/4\utime$, and with probaility $1-10 d \utime^2 T^2 \cdot \log(\uspace\sqrt{d}/(\eta r)) e^{-\logt}$:
\begin{equation*}
f(\x_{T}) - f(\x_0) \le  -(0.9M - c\sqrt{M \cdot \logt})\ufun  +c\cdot\eta \tilde{\sigma}^2 (\eta \ell T + 2M\logt)
\le -0.4 M \ufun \le -0.4 T\ufun/\utime.
\end{equation*}
This will gives $f(\x_T) \le f(x_0) -0.4 T\ufun/\utime < f^\star$ which is not possible.

Finally, it is not hard to verify, by choosing $\logt = c\cdot\log \left( \frac{ d \ell \Delta_f \Ns}{\rho\epsilon\delta}\right)$ for a large enough value of the absolute constant $c$, we can make both claims hold with probability $1-\delta$.
\end{proof}

\subsection{Proof of Theorem \ref{thm:main_SGD_minibatch}}

Our proofs for PSGD easily generalize to the mini-batch setting.

\begin{proof}[Proof of Theorem \ref{thm:main_SGD_minibatch}]
The proof is essentially the same as the proof of Theorem \ref{thm:main_SGD}. The only difference is that, up to a log factor, mini-batch PSGD reduces the variance $\sigma^2$ and $\tilde{\ell}^2\norm{\dif{\x}_{\tau}}^2$ in Eq.~\eqref{eq:SGlipcoupling} by a factor of $m$, where $m$ is the size of the mini-batch.
\end{proof}

\section{Concentration Inequalities}

In this section, we present the concentration inequalities required for this paper.
Please refer to the technical note \citep{jinshortnote} for the proofs of Lemmas~\ref{lem:examplesGnorm}, \ref{lem:nSGvariant}, \ref{lem:concen_sum} and \ref{lem:concen_sum_randB}.

Recall the definition of a norm-subGaussian random vector.
\begin{definition}\label{def:sGnorm}
A random vector $\X \in \R^d$ is \emph{norm-subGaussian} (or $\nSG(\sigma)$), if there exists $\sigma$ so that:
\begin{equation*}
\Pr\left(\norm{\X - \E \X} \ge t \right) \le 2 e^{-\frac{t^2}{2\sigma^2}}, \qquad \forall t \in \R.
\end{equation*}
\end{definition}

Note that a bounded random vector and a subGaussian random vector are two special cases of a norm-subGaussian random vector.
\begin{lemma}\label{lem:examplesGnorm}
There exists an absolute constant $c$ so that following random vectors are  $\nSG(c \cdot \sigma)$.
\begin{enumerate}
\item A bounded random vector $\X \in \R^d$ such that $\norm{\X} \le \sigma$.
\item A random vector $\X \in \R^d$, where $\X = \xi \e_1$ and the random variable $\xi \in \R$ is $\sigma$-subGaussian.
\item A random vector $\X \in \R^d$ that is $(\sigma/\sqrt{d})$-subGaussian.
\end{enumerate}
\end{lemma}

Second, we have that if $\X$ is norm-subGaussian, then its norm square is subExponential, and its component along a single direction is subGaussian.
\begin{lemma}\label{lem:nSGvariant}
There is an absolute constant $c$ so that if the random vector $\X \in R^d$ is zero-mean $\nSG(\sigma)$, then $\norm{\X}^2$ is $c\cdot\sigma^2$-subExponential, and for any fixed unit vector $\v \in \S^{d-1}$, $\la \v, \X\ra$ is $c\cdot\sigma$-subGaussian.
\end{lemma}

For concentration, we are interested in the properties of norm-subGaussian martingale difference sequences. Concretely, they are sequences satisfying the following conditions.

\begin{condition} \label{cond:subGmartingale}
Consider random vectors $\X_1, \ldots, \X_n \in \R^d$, and corresponding filtrations $\F_i = \sigma(\X_1, \ldots, \X_i)$ for $i\in [n]$, such that $\X_i |\F_{i-1}$ is zero-mean $\nSG(\sigma_i)$ with $\sigma_i \in \F_{i-1}$. That is:
\begin{equation*}
\E [\X_i |\F_{i-1}]  = 0, \quad \Pr\left(\norm{\X_i} \ge t | \F_{i-1}\right) \le 2 e^{-\frac{t^2}{2\sigma_i^2}}, \qquad \forall t \in \R, \forall i \in [n].
\end{equation*}
\end{condition}

Similar to subGaussian random variables, we can also prove a Hoeffding-type inequality for norm-subGaussian random vectors which is tight up to a $\log(d)$ factor.

\begin{lemma} [Hoeffding-type inequality for norm-subGaussian] \label{lem:concen_sum}
 Given $\X_1, \ldots, \X_n \in \R^d$ that satisfy condition \ref{cond:subGmartingale},  with fixed $\{\sigma_i\}$, then for any $\logt >0$, there exists an absolute constant $c$ such that, with probability at least $1-2d \cdot e^{-\logt}$:
\begin{equation*}
\norm{\sum_{i=1}^n \X_i}  \le c\cdot \sqrt{\sum_{i=1}^n \sigma_i^2 \cdot \logt}.
\end{equation*}
\end{lemma}
When $\{\sigma_i\}$ is also random, we have the following.

\begin{lemma} \label{lem:concen_sum_randB}
Given $\X_1, \ldots, \X_n \in \R^d$ that satisfy condition \ref{cond:subGmartingale}, then for any $\logt >0$, and $B > b > 0$, there exists an absolute constant $c$ such that, with probability at least $1-2d \log (B/b) \cdot e^{-\logt}$:
\begin{equation*}
\sum_{i=1}^n \sigma_i^2 \ge B \quad \text{or} \quad
\norm{\sum_{i=1}^n \X_i}  \le  c\cdot \sqrt{\max\{\sum_{i=1}^n \sigma_i^2, b\}\cdot \logt}.
\end{equation*}
\end{lemma}

Finally, we can also provide concentration inequalities for the sum of norm squares of norm-subGaussian random vectors, and for the sum of inner products of norm-subGaussian random vectors with another set of random vectors.

\begin{lemma}\label{lem:concen_square}
Given $\X_1, \ldots, \X_n \in \R^d$ that satisfy Condition \ref{cond:subGmartingale} with fixed $\sigma_1 = \ldots = \sigma_n = \sigma$, then there exists an absolute constant $c$ such that, for any $\logt >0$, with probability at least $1-e^{-\logt}$:
\begin{equation*}
\sum_{i=1}^n\norm{\X_i}^2 \le c \cdot \sigma^2 \left(n + \logt\right).
\end{equation*}
\end{lemma}

\begin{proof} Note there exists an absolute constant $c$ such that $\E[\norm{\X_i}^2 |\F_{i-1}] \le c\cdot \sigma^2$, and $\norm{\X_i}^2 |\F_{i-1}$
is $c \cdot \sigma^2$-subExponential. This lemma directly follows from standard Bernstein concentration inequalities for subExponential random variables.
\end{proof}

\begin{lemma}\label{lem:concen_inner}
 Given $\X_1, \ldots, \X_n \in \R^d$ that satisfy Condition \ref{cond:subGmartingale} and random vectors $\{\u_i\}$ that satisfy $\u_i \in \F_{i-1}$ for all $i\in [n]$, then for any $\logt >0$, $\lambda >0$, there exists absolute constant $c$ such that, with probability at least $1-e^{-\logt}$:
\begin{equation*}
\sum_{i} \la \u_i, \X_i \ra \le c \cdot \lambda\sum_i \norm{\u_i}^2\sigma_i^2 + \frac{1}{\lambda}\cdot \logt.
\end{equation*}
\end{lemma}

\begin{proof}
For any $i \in [n]$ and fixed $\lambda > 0$, since $\u_i \in \F_{i-1}$, according to Lemma \ref{lem:nSGvariant} there exists a constant $c$ such that $\la \u_i, \X_i \ra |\F_{i-1}$ is $c\cdot \norm{\u_i}\sigma_i$-subGaussian. Thus:
\begin{equation*}
\E[ e^{\lambda \la \u_i, \X_i\ra} | \F_{i-1}] \le e^{c\cdot \lambda^2 \norm{\u_i}^2 \sigma_i^2}.
\end{equation*}
Therefore, consider the following quantity:
\begin{align*}
\E e^{\sum_{i=1}^t  (\lambda \la \u_i, \X_i\ra - c \cdot\lambda^2 \norm{\u_i}^2 \sigma_i^2)} 
&= \E \left[e^{\sum_{i=1}^{t-1} \lambda \la \u_i, \X_i\ra - c \cdot\sum_{i=1}^{t}\lambda^2 \norm{\u_i}^2 \sigma_i^2}   
\cdot \E\left( e^{\lambda \la \u_t, \X_t\ra} | \mathcal{F}_{t-1}\right) \right]\\
&\le  \E \left[e^{\sum_{i=1}^{t-1} \lambda \la \u_i, \X_i\ra - c \cdot\sum_{i=1}^{t}\lambda^2 \norm{\u_i}^2 \sigma_i^2}   
\cdot e^{c \cdot\lambda^2 \norm{\u_t}^2 \sigma_t^2} \right] \\
&= \E e^{\sum_{i=1}^{t-1}  (\lambda \la \u_i, \X_i\ra - c \cdot\lambda^2 \norm{\u_i}^2 \sigma_i^2)}  \le 1.
\end{align*}
By Markov's inequality, for any $t > 0$:
\begin{align*}
\Pr\left(\sum_{i=1}^t  (\lambda \la \u_i, \X_i\ra - c \cdot\lambda^2 \norm{\u_i}^2 \sigma_i^2) \ge t\right) &\le
\Pr\left( e^{\sum_{i=1}^t  (\lambda \la \u_i, \X_i\ra - c \cdot\lambda^2 \norm{\u_i}^2 \sigma_i^2)} \ge e^t\right) \\
&\le e^{-t} \E e^{\sum_{i=1}^t  (\lambda \la \u_i, \X_i\ra - c \cdot\lambda^2 \norm{\u_i}^2 \sigma_i^2)} \le e^{-t}.
\end{align*}
This finishes the proof.
\end{proof}

\end{document}